%% file: arxiv-bdb-22.tex
\documentclass[12pt]{article}

\date{}

\usepackage[T1]{fontenc}
\usepackage{lmodern}
\usepackage[utf8]{inputenc} 
\usepackage{url}            
\usepackage{booktabs}       
\usepackage{amsfonts}       
\usepackage{microtype}      
\usepackage{xcolor}         
\usepackage{wrapfig}
\usepackage{enumitem}
\usepackage{fullpage}
\usepackage{graphicx}
\usepackage{subfigure}
\usepackage{natbib}

\input{macros.tex}

\title{\papertitle}

\author{
Aadirupa Saha%
\thanks{Microsoft Research, NYC. {\tt aadirupa.saha@microsoft.com}.}
\and 
Pierre Gaillard \thanks{Univ. Grenoble Alpes, Inria, CNRS, Grenoble INP, LJK, 38000 Grenoble, France. {\tt pierre.gaillard@inria.fr}}
}

\begin{document}

\maketitle

\input{abstract.tex}

\input{introduction.tex}

\input{problem.tex}

\input{algo_rr.tex}

\input{reduction.tex}

\input{algo_bdb.tex}

\input{corrupted_regime.tex}

\input{expts_arxiv.tex}

\input{conclusion.tex}

\section*{Acknowledgment}
Thanks to Julian Zimmert and Karan Singh for the useful discussions on the existing best-of-both-world multiarmed bandits results.

\newpage
\bibliographystyle{plainnat}
\bibliography{bib_bdb}

\newpage
\appendix
\clearpage
\input{appendix.tex}

\input{algo_analysis.tex}

\end{document}

%% file: macros.tex
\usepackage{microtype}
\usepackage{amssymb}
\usepackage{pifont}
\usepackage{nicefrac}
\usepackage{cancel}
\usepackage{array}
\usepackage{graphicx}
\usepackage{dsfont}
\usepackage{booktabs} 
\usepackage{amssymb}
\usepackage{color}
\usepackage{amsmath}
\usepackage{amsthm}
\usepackage{thmtools}
\usepackage{thm-restate}
\usepackage{algorithm}
\usepackage{algorithmic}
\usepackage{enumitem}
\usepackage{setspace}
\usepackage{hyperref}
\usepackage{nameref}

\newtheorem{thm}{Theorem}
\newtheorem*{thm*}{Theorem}
\newtheorem*{lem*}{Lemma}
\newtheorem{corollary}{Corollary}

\newtheorem{rem}{Remark}

\newcommand{\R}{{\mathbb R}}

\renewcommand{\P}{{\mathbf P}}

\newcommand{\E}{{\mathbf E}}

\newcommand{\1}{{\mathbf 1}}
\newcommand{\0}{{\mathbf 0}}

\newcommand{\cA}{{\mathcal A}}

\newcommand{\cI}{{\mathcal I}}

\newcommand{\M}{{\mathbf M}}

\newcommand{\hp}{{\hat p}}

\newcommand{\p}{{\mathbf p}}

\newcommand{\y}{{\mathbf y}}

\newcommand{\sm}{\setminus}

\renewcommand{\epsilon}{\varepsilon}
\renewcommand{\hat}{\widehat}

\newcommand{\bmpi}{\boldsymbol {\pi}}

\newcommand{\btheta}{\boldsymbol \theta}

\newcommand{\bsigma}{\boldsymbol \sigma}

\def \algbdb{\texttt{Versatile-DB}} 
\def \algrr{\texttt{RR-DB}} 

\def \bdb{\texttt {Versatile Dueling Bandits}}
\def \papertitle{Versatile Dueling Bandits: Best-of-both-World Analyses for Online Learning from Preferences}

\makeatletter
\renewcommand{\paragraph}{%
  \@startsection{paragraph}{4}%
  {\z@}{0.1ex \@plus .5ex \@minus .1ex}{-1em}%
  {\normalfont\normalsize\bfseries}%
}
\makeatother


%% file: abstract.tex
\begin{abstract}

We study the problem of $K$-armed dueling bandit for both stochastic and adversarial environments, where the goal of the learner is to aggregate information through relative preferences of pair of decisions points queried in an online sequential manner. 
We first propose a novel reduction from any (general) dueling bandits to multi-armed bandits and despite the simplicity, it allows us to improve many existing results in dueling bandits. In particular, \emph{we give the first best-of-both world result for the dueling bandits regret minimization problem}---a unified framework that is guaranteed to perform optimally for both stochastic and adversarial preferences simultaneously.
%
%
Moreover, our algorithm is also the first to achieve an optimal $O(\sum_{i = 1}^K \frac{\log T}{\Delta_i})$ regret bound against the Condorcet-winner benchmark, which scales optimally both in terms of the arm-size $K$ and the instance-specific suboptimality gaps $\{\Delta_i\}_{i = 1}^K$. This resolves the long standing problem of designing an instancewise gap-dependent order optimal regret algorithm for dueling bandits (with matching lower bounds up to small constant factors).
We further justify the robustness of our proposed algorithm by proving its optimal regret rate under adversarially corrupted preferences---this outperforms the existing state-of-the-art corrupted dueling results by a large margin.
In summary, we believe our reduction idea will find a broader scope in solving a diverse class of dueling  bandits setting, which are otherwise studied separately from multi-armed bandits with often more complex solutions and worse guarantees. 
%
%
The efficacy of our proposed algorithms is empirically corroborated against the existing dueling bandit methods. 
\end{abstract}
\vspace{-10pt}


%% file: introduction.tex
\section{Introduction}
\label{sec:intro}





Studies have shown that it is often easier, faster and less expensive to collect feedback on a relative scale rather than asking ratings on an absolute scale. E.g., to understand the liking for a given pair of items, say (A,B), it is easier for the users to answer preference-based queries like: ``Do you prefer Item A over B?", rather than their absolute counterparts: ``How much do you score items A and B in a scale of  [0-10]?".
From a system designer's point of view, exploiting such user preference information could greatly aid in improving systems performances,
especially when data can be collected on a relative scale and online fashion; such as recommendation systems, crowd-sourcing platforms, training bots, multi-player games, search-engine optimization, online retail, just to name a few.
In many real world problems, especially where human preferences are elicited in an online fashion, e.g., design of surveys and expert reviews, assortment selection, search engine optimization, recommender systems, ranking in multiplayer games, etc, or even more general reinforcement learning problems where rewards shaping is often a challenging problem (e.g. if multi-objective rewards etc.), and instead, a preference feedback is much easier to elicit.

Due to the widespread applicability and ease of data collection with relative feedback, learning from preferences has gained much popularity in the machine learning community and widely studied as the problem of \emph{Dueling-Bandits} (DB) over last decade \cite{Yue+12,ailon2014reducing,Zoghi+14RUCB,Zoghi+14RCS,Zoghi+15}, which is an online learning framework that generalizes the standard multiarmed bandit (MAB) \cite{Auer+02} setting for identifying a set of `good' arms from a fixed decision-space (set of items) by querying preference feedback of actively chosen item-pairs. 

\paragraph{Dueling Bandits Problem (DB)} More formally, in classical dueling bandits with $K$ arms, the learning proceeds in rounds, where at each time step $t \in \{1,2,\ldots,T\}$, the learner selects a pair of arms $(k_{+1,t},k_{-1,t})$ and receives the winner of the duel in terms of a binary preference feedback $o_t(k_{+1,t},k_{-1,t}) \sim \text{Ber}(P_t(k_{+1,t},k_{-1,t}))$, sampled according to an underlying preference matrix $\P_t \in [0,1]^{K \times K}$ (chosen adversarially in the most general setup). The objective of the learner is to minimize the regret with respect to a (or set of) `best' arm(s) in hindsight.  

\paragraph{Related Works} Over the years, the problem of Dueling Bandits has been studied with various objectives and generalizations. This includes analyzing the learning rate under various preference structures, such as total ordering, transitivity, stochastic triangle inequality \cite{falahatgar_nips,BTM}, utility based preference structure \cite{ailon2014reducing,Busa_pl,SG18,ChenSoda+18}, or under any general pairwise preference matrices \cite{CDB,SparseDB,Komiyama+16}. Consequently, depending on the underlying preference structure, the learner's performance has been evaluated w.r.t. different benchmarks including 
among other promising generalizations. The problem of stochastic dueling bandits has been studied for both PAC \cite{falahatgar2,Busa_pl,sui2018advancements} as well as regret minimization setting \cite{Zoghi+14RUCB,Yue+09,WeakDB,Zoghi+15} under several notions of benchmarks including best arm identification \cite{SGwin18,Busa_aaai,falahatgar_nips}, top-set detection \cite{Busa_top,SG19}, ranking \cite{Ren+18,SGrank18}, amongst many. Some recent works have also looked into the problem for adversarial sequence of preference matrices \cite{gupta2021optimal,Adv_DB,ADB}, robustness to corruptions~\cite{agarwal2021stochastic}, or extending dueling bandits to potentially infinite arm sets \cite{S21,ContDB} and contextual scenarios \cite{CDB,SK21}. Another interesting line of work along dueling bandits is to study the implications for its subsetwise generalization \cite{Ren+18,sui2018advancements,Brost+16,ChenSoda+18}, also studied as battling bandits \cite{SG18,SG19,bengs2021preference}.

Despite widespread surge of interest along this line of research and multiple attempts there are some fundamental long standing open questions in dueling bandits which are (surprisingly!) yet unresolved. 
 
\vspace{3pt} 
\paragraph{Unresolved Question \#1} 
One of the longest and most widely studied objective in stochastic dueling bandit is to measure regret w.r.t. the Condorcet winner (CW) arm: Given a preference matrix $P$, an arm $k^{\texttt{(cw)}} \in [K]$ is termed as the CW of $P$ if $P(k^{\texttt{(cw)}},k) > 0.5 ~\forall k \in [K]\setminus \{k\}$ \cite{Zoghi+14RUCB}. 
Assuming $P$ contains a CW, there have been several attempts to design an optimal regret dueling bandit algorithm against the CW arm $k^{\texttt{(cw)}}$ (see Eq.~\eqref{eq:cwreg} for details) \cite{Zoghi+14RUCB,DTS,Komiyama+15,bengs2021preference}. 
Without loss of generality, assuming $k^{\texttt{(cw)}}=1$ and by denoting $\Delta_{i} = P(1,i)-0.5$ to be the suboptimality gap of item $i$ w.r.t. the CW, it is well known that the dueling bandit regret lower bound (w.r.t. the CW arm) is $\Omega\big( \sum_{i = 2}^K \frac{\log T}{\Delta_i} \big)$ \cite{Yue+12,Komiyama+15}. However, despite several attempts, it is still unknown how to design an order optimal dueling bandit algorithm for the CW regret. Existing upper-bounds suffer all suboptimal $\Delta_{\min}^{-2}$ and/or $K^2$ dependencies. 

Notably, under more restricted structures, e.g. total ordering \cite{BTM}, or utility based preferences \cite{Busa_pl,SGinst20}, or even special preference structures where the suboptimality gaps of all items ($\Delta_i$, $~i \in [K]\sm \{1\}$) are equal, the problem is easier to solve and tight regret guarantees are available with matching upper and lower bound analysis. However, for the case of any general preference matrix with CW, none of the existing attempts were able to close this regret analysis gap successfully \cite{Zoghi+14RUCB,DTS,WeakDB,Komiyama+15,SG21dbaa}.
Subsequently, the natural questions to ask are: 
\begin{center}
\emph{(1). Is the lower bound tight? (2). How to close the gap between the upper and lower bound for CW regret?}
\end{center}

\paragraph{Unresolved Question \#2} Till date, all the proposed algorithms of dueling bandits need to know underlying preference structure ahead of time in order to yield optimal regret bounds. In fact, different algorithms have been proposed based on the nature/structures of the underlying preference matrices, e.g. \cite{BTM} for preferences with total orderings in terms of (relaxed) stochastic transitivity and stochastic triange inequality; \cite{ailon2014reducing,Adv_DB} for linear-utility based preferences, \cite{Busa_pl} for BTL models, \cite{Zoghi+14RUCB,Komiyama+15,DTS} for stochastic preferences in presence of CW, \cite{ADB,Adv_DB} for adversarial sequence of preferences, etc. However, it might not always be realistic to assume complete knowledge of the properties underlying preference matrices. 
Thus the daunting question to ask in this regard is
\begin{center}
	\emph{Is it possible to design an order optimal `best-of-both-world' algorithm for dueling bandits?}
\end{center}
That is, an algorithm that can adapt itself to the underlying structures of the preference environments and give optimal regret for both stochastic and adversarial settings? There has been a series of work on this line for the MAB framework \citep[e.g.,][]{bubeck2012best,auer2016algorithm,zimmert2021tsallis}, but unfortunately there has not been any existing `best-of-both-world' attempt for general dueling bandits.

\vspace{3pt} 
\paragraph{Unresolved Question \#3} In any real world situation, the true  feedback are often corrupted with some form of system noise. Undoubtedly, the binary $0/1$ bit dueling preferences are extremely prone to such noises when the learner might get to observe a flipped bit (adversarially corrupted) instead of the true dueling feedback. 
\begin{center}
	\emph{Can we design an efficient dueling bandit algorithm which is robust to adversarial corruptions and provably optimal?}
\end{center}

\paragraph{Our Contributions} 
In this paper, we answer all of the above three questions affirmatively. The list of our specific contributions can be summarized as follows:
\begin{enumerate}[ topsep=-\parskip]
\item \textbf{A novel insight on the reduction from MAB to DB.} \citet{ailon2014reducing} proposed a reduction from MAB to utility-based dueling bandits. We show that the latter can easily be extended to more general dueling bandit problems (including CW), with significant consequences (see below) on the state of the art in dueling bandits theory. We believe that the reduction will find wider application in solving a diverse class of dueling bandit settings, using analyses of their MAB counterparts, which are otherwise studied separately from MAB with often more complex solutions and worse guarantees.

\item \textbf{First Best-of-Both World regret for DB.} Applying the above reduction to a Best-of-Both world algorithm from MAB, we provide an algorithm that simultaneously guarantees a pseudo-regret bound $O(\sqrt{KT})$ in the adversarial setting and $O(K \log(T)/\Delta_{\min})$ in the stochastic one.

\item \textbf{Robustness to adversarial corruptions.} Our algorithm is robust to adversarial corruptions and significantly improves existing results in DB  \cite{agarwal2021stochastic}. 

\item \textbf{Optimal stochastic gap-dependent Regret.} Our algorithm also provides the first optimal Condorcet regret, which suffers neither from a suboptimal dependence on $\Delta_{\min}^{-2}$ nor from a quadratic dependence on the number of arms.

\item \textbf{Another easy algorithm for stochastic DB.} We also propose a new very simple elimination algorithm with $O(\sum_{i = 2}^K\frac{K\log T}{\Delta_i})$ Condorcet regret.

\item \textbf{Experimental evaluations. }  Finally we corroborate our theoretical results with extensive empirical evaluations  (Sec.\,\ref{sec:expts}).

\end{enumerate}



%% file: problem.tex


\section{Problem Formulation}
\label{sec:prob}
\textbf{Notations.} Decision space (or item/arm set) $[K]: = \{1,2,\ldots, K\}$. 
For any matrix $\M \in \R^{K \times K}$, we define $m_{ij} := M(i,j),~\forall i,j \in [K]$. 
$\1(\cdot)$ denotes the indicator random variable which takes value $1$ if the predicate is true and $0$ otherwise and $\lesssim$ a rough inequality which holds up to universal constants. For any two items $x,y \in [K]$, we use the symbol $x \succ y$ to denote $x$ is preferred over $\y$. 
By convention, we set $\frac{0}{0}:=0.5$.

\noindent \textbf{Setup. } We assume a decision space of $K$ arms denoted by $\cA:= [K]$. At each round $t$, the task of the learner is to select a pair of actions $(k_{+1,t},k_{-1,t}) \in [K]\times [K]$, upon which a preference feedback $o_t \sim \text{Ber}(P_t(k_{+1,t},k_{-1,t}))$ is revealed to the learner according to the underlying preference matrix $\P_t \in [0,1]^{K \times K}$ (chosen adversarially), such that the probability of $k_{+1,t}$ being preferred over $k_{-1,t}$ is given by $Pr(o_t = 1):= Pr(k_{+1,t} \succ k_{-1,t}) = P_t(k_{+1,t},k_{-1,t})$, and hence $Pr(o_t = 0):= Pr(k_{-1,t} \succ k_{+1,t}) = 1-P_t(k_{+1,t},k_{-1,t})$.
%

\noindent \textbf{Objective. }
Assuming the learner selects the duel $(k_{+1,t},k_{-1,t})$ at round $t$, one can measure its performance w.r.t. a single fixed arm $k^* \in [K]$\footnote{Note that this is equivalent to maximizing the expected regret w.r.t. any fixed distribution $\bmpi^* \in \Delta_K$, i.e. when $k^* \sim \bmpi^*$. This is because the regret objective is linear in the entries of $\bmpi^*$, so the maximizer $\bmpi^*$ is always one hot.} in hindsight by calculating the regret w.r.t. $k^* \in [K]$
\begin{equation}
	\label{eq:sreg}
	R_T(k^*) :=  \sum_{t=1}^T  \frac{1}{2}\left( P_t(k^*, k_{+1,t}) + P_t(k^*, k_{-1,t}) -1 \right). 
\end{equation}

For the \emph{stochastic setting} where $P_t$s are fixed across all time steps $t \in [T]$, we denote $P_t = P ~\forall t\in [T]$. Further assuming there exists a Condorcet winner for $P$, i.e. fixed arm $k^{\texttt{(cw)}} \in [K]$ such that $P(k^{\texttt{(cw)}},k) > 0.5 ~\forall k \in [K]\setminus \{k\}$, the above notion of regret boils down to the regret with respect to the Condorcet winner for the choice of $k^* = k^{\texttt{(cw)}}$, as widely studied in many dueling bandit literature \cite{Zoghi+14RUCB,DTS,Komiyama+15,bengs2021preference}, defined as:
\begin{equation}
	\label{eq:cwreg}
	R_T^{\texttt{(cw)}} :=  \sum_{t=1}^T  \frac{1}{2}\left( \Delta(k^{\texttt{(cw)}}, k_{+1,t}) + \Delta(k^{\texttt{(cw)}}, k_{-1,t}) \right), 
\end{equation}
where $\Delta(i,j):= P(i,j)-1/2$ being the suboptimality gap of item $i$ and $j$ in terms of their relative preferences.



%% file: algo_rr.tex
\section{Warm-Up: Near-Optimal Algorithm}
\label{sec:algo_rr}
In this section, we propose a new simple UCB based algorithm for stochastic dueling bandit, which is shown to have a nearly optimal gap-dependent Condorcet regret of $R_T^{\texttt{(cw)}} = O\big( \sum_{i=2}^K { i {\Delta(k^{\texttt{(cw)}},k)^{-1}}\log T}\big)$. Note that existing dueling bandit algorithms, that satisfy a non-asymptotic Condorcet regret bound, suffer an additional constant of order $\Delta_{\min}^{-2}$ \citep{bengs2021preference}, which implies a worst-case regret of order $O(T^{2/3})$ when $\Delta_{\min} \to 0$. The simple elimination algorithm below solves this drawback and depends only on $\Delta_{\min}^{-1}$ but at the cost of a suboptimal quadratic dependence in the number of arms $K$. Despite our efforts, we could not avoid this suboptimal factor by following the classical stochastic dual bandit analysis. In the following sections, we will show how to easily reach the optimal dependence in both $\Delta_{\min}$ and $K$ using a simple reduction from standard MAB.


\textbf{Main Ideas: Algorithm \algrr} The high-level idea of Algorithm~\ref{alg:rr} is to sequentially compare arms in a round-robin fashion and eliminate arms when they are significantly suboptimal compared to any other arm. Typically, after $t$ rounds, a suboptimal arm $k$ has been compared at least $t/K$ times with the Condorcet winner. Denoting by $\Delta_k$ its suboptimality gap, the arm is eliminated after at most $t_k$ rounds, where $(t_k/K)^{-1/2} \approx \Delta_k$. At that time the arm has been played $t_k/K$ times, yielding a regret of order $(t_k/K) \times \Delta_k \approx (K/\Delta_k^2) \times \Delta_k = K/\Delta_k$. Summing over the arms yields a final regret of order $O(K^2/\Delta_{\min})$. 
\begin{algorithm}[h]
	\caption{\textbf{\algrr\, (Near Optimal DB)}}
	\label{alg:rr}
	
	\begin{algorithmic}[1]	
		\STATE {\bfseries input:} Arm set: $[K]$, Confidence parameter $\delta \in (0,1)$ 
		\STATE {\bfseries init:} Active arms: $\cA_1 := [K]$, $n_{ij}(t) \leftarrow 0, ~\forall i,j \in [K]$ 
		\FOR{$t = 1, 2, \ldots, T$}
		\STATE Play $(k_{+1,t},k_{-1,t}) \in \text{argmin}_{i,j \in \cA_t} \{ n_{ij}(t-1)\}$ \\
		\STATE Observe $o_t(k_{+1,t},k_{-1,t}) = 1 - o_t(k_{+1,t},k_{-1,t})$
		\FOR{$i,j \in \cA_t$}
		    \STATE Define $\mathds{1}_t(i,j) := \mathds{1}\big\{\{i,j\} = \{k_{-1,s},k_{+1,s}\}\big\}$ and  \\
	    	\qquad $n_{ij}(t) := \sum_{s=1}^t \mathds{1}_t(i,j)$ \\
		    \qquad $\hat p_{ij}(t) := \frac{1}{n_{ij}(t)} \sum_{s=1}^t o_t(i,j) \mathds{1}_t(i,j)$  \\
		    \qquad $u_{ij}(t) := \hat p_{ij}(t) + \sqrt{\frac{\log (Kt/\delta)}{n_{ij}(t)}}$ \\
		    where we assume $x/0 = +\infty$.
		\ENDFOR
		\STATE $\cA_{t+1} := \cA_t \backslash \big\{i \in \cA_t: \exists j \in \cA_t \ \text{s.t.}\quad  u_{ij}(t) < \frac{1}{2}\big\}$ \,.	
		\ENDFOR  
	\end{algorithmic} 

\end{algorithm} 


Without loss of generality assume the Condorcet winner $k^{\texttt{(cw)}} = 1$, and denote $\Delta_i = \Delta(1,i), ~\forall i \in [K]\sm \{1\}$.

\begin{restatable}[]{thm}{thmrr}
	\label{thm:rr}
	Let $\delta \in (0,1/2)$, for any $T\geq 1$, the regret of Algorithm~\ref{alg:rr} is upper-bounded with probability at least $1-\delta$ as
	\begin{equation*}
		{R}_T^{\texttt{(cw)}} \leq \frac{K^2}{2} +  4 \sum_{i=2}^K \frac{ (i-1) \log (KT/\delta)}{\Delta_i}.
	\end{equation*}
	Further, when $T\geq K^2$, in the worst case (over the problem instance, i.e. $\Delta_2,\ldots,\Delta_K$), the regret of Algorithm~\ref{alg:rr} can be upper bounded as:
	\[
			{R}_T^{\texttt{(cw)}} \leq 2 K \sqrt{T \log(KT/\delta) }.
	\]
\end{restatable}

\begin{rem}
	 In particular, our regret analysis shows that, except a logarithmic factor, the regret bound of \algrr~(Alg.~\ref{alg:rr}) is off only by a multiplicative factor of $K$, as follows from the known $\Omega(\sum_{k=1}^K \frac{\log T}{\Delta_k})$ Condorcet winner regret lower bound \cite{Yue+12,Komiyama+15}. 
\end{rem}

The proof is postponed to Appendix~\ref{app:proofs}.

%% file: reduction.tex
\section{Reduction from MAB to DB}

We now present a simple reduction from a multi-armed bandit algorithm to a dueling bandit one. The latter was already proposed by \cite{gupta2021optimal} to show worst-case guarantees and  by~\cite{ailon2014reducing} for utility based dueling bandits only. We recall it here since it is central to our analysis and we believe that it is of important interest for the dueling bandit community that usually uses significantly different algorithms and analysis than the ones from standard multi-armed bandits.

The main idea is to apply the multi-armed bandit algorithm  independently to two players $i \in \{-1,+1\}$ respectively with losses defined for any $k \in [K]$ and $t \in [T]$, by
\[
    \ell_{i,t}(k) := o_t(k_{-i,t},k) \,,
\]
where $o_t(k,k') = 1-o_t(k',k)$ for $1<k'\leq K$ follows a Bernoulli with parameter $P_t(k,k')$ (and we assume $o_t(k,k) = 1/2$).

\begin{algorithm}[h]
   \caption{Reduction from MAB to DB}
   \label{alg:reduction}
\begin{algorithmic}[1]	
   \STATE {\bfseries input:} Arm set: $[K]$, two instances $\cA_i$ of an algorithm for MAB,  $i\in \{-1,+1\}$. 
    \FOR{$t = 1, 2, \ldots, T$}
        \FOR{$i \in \{+1, -1\}$} 
    		\STATE  choose $\p_{i,t}$ from $\cA_i$
    		\STATE  sample $k_{i, t}$ from the distribution $p_{i,t}$
	    \ENDFOR
		\STATE Observe preference feedback $o_t(k_{+1, t}, k_{-1, t})$ and set $o_t(k_{-1,t},k_{+1,t}) = 1- o_t(k_{+1, t}, k_{-1, t})$.
		\STATE Feed $\cA_i$ with loss $\ell_{i,t}(k) := o_t(k_{-i,t},k_{i,t})$ for $i \in \{-1,+1\}$.
   \ENDFOR  
\end{algorithmic} 
\end{algorithm} 

We show below that any MAB regret upper-bound satisfied by $\cA_i$ can be transformed into a DB regret upper-bound.

\begin{thm}
\label{thm:reduction}
Define for $i\in \{-1,1\}$ and $k \in [K]$ by 
\[
    R_{i,T}(k) := \sum_{t=1}^T \ell_{i,t}(k_{i,t}) - \ell_{i,t}(k) 
\]
the regret achieved by algorithm $\cA_i$.  Then, the expected regret~\eqref{eq:sreg} of Algorithm~\ref{alg:reduction} for dueling bandits can be decomposed as
\[
    \E\big[R_T(k)\big] = \frac{1}{2} \E\big[ R_{-1,T}(k) + R_{+1,T}(k)  \big] \,.
\]
\end{thm}

\begin{proof}
The proof follows from
\begin{align*}
    \E\big[ & \ell_{-1,t}(k)  + \ell_{+1,t}(k)\big] 
         = \E\big[o_t(k_{+1, t},k) + o_t(k_{-1,t},k)\big] \\
        & = \E\big[P_t(k_{+1, t},k) + P_t(k_{-1,t},k)\big]\\
        & = 2 - \E\big[P_t(k, k_{+1, t}) + P_t(k, k_{-1,t})\big]
\end{align*}
and
\begin{align*}
    \E\big[ &\ell_{-1,t}(k_{+1,t})  +  \ell_{+1,t}(k_{-1,t})\big] \\
        & = \E\big[o_t(k_{+1, t}, k_{-1, t}) + o_t(k_{+1, t}, k_{-1, t})\big] 
        = 1 \,.
\end{align*}
We conclude by summing over $t=1,\dots,T$ both equations and by substituting them into the definition of the regret $R_T(k)$ in~\eqref{eq:sreg}.
\end{proof}

Note that such a reduction can also be used to bound $R_T(k)$ directly rather than its expectation. \citet{gupta2021optimal} indeed use this reduction to show a $O(\sqrt{KT})$ high-probability regret upper-bound for adversarial dueling bandit. They also obtain non-stationary regret bounds. 

The main message of this paper is that this reduction can be used to transpose many results from standard multi-armed bandit to general dueling bandits. For instance, applying a subroutine $\cA_i$ which is robust to delays \citep[e.g.,]{delay3,zimmert2020optimal}, one directly obtains a dueling bandit with the same robustness guarantees.

As we said, this reduction is not new. However, to date, it has only been considered in two specific contexts: the adversarial setting with worst-case regret bounds of order $O(\sqrt{KT})$ and the utility-based setting. Since the losses $\ell_{i,t} = o_t(k_{-i,t},k)$ are not i.i.d. but depend on an adaptive adversary which chooses $k_{-i,t}$, one cannot use stochastic bandit algorithms. And the dueling bandit community usually needs to resort to more sophisticated algorithms and arguments to obtain logarithmic regret bounds for Condorcet stochastic dueling bandits. The only stochastic dueling bandit for which such a reduction was considered \citep[e.g.,]{ailon2014reducing,zimmert2021tsallis} was the utility based-dueling bandits, which is overly restrictive in practice. That is, when the preference matrix if of the form $P_t(k,k') : (1 + u_t(k) + u_t(k'))/2$ for some sequence of utility vectors $(u_t)_{t\geq 1}$. Utility based dueling bandit are known to be easily be reduced to two independent multi-armed bandit problems \citep{ailon2014reducing}. 

Our main contribution is to show that this reduction can in fact be easily extended to the much weaker Condorcet winner hypothesis. To do this, as we show in the next sections, we simply apply the reduction with a best-of-both-worlds multi-armed bandit algorithm. As we will see, this recovers and improves the best existing upper bounds on the Condorcet regret for dueling bandits.

%% file: algo_bdb.tex
\section{Best-of-Both Dueling: Optimal Algorithm for Stochastic and Adversarial DB}
\label{sec:algo_bdb}

This section contains our main result, which is a simple reduction of the best-of-both-worlds result from~\cite{zimmert2021tsallis} to dueling bandits. In particular, it allows us to improve the best existing upper bound on the regret for stochastic and corrupted dueling bandits. The main idea is to apply the reduction of Algorithm~\ref{alg:reduction} with the multi-armed bandit algorithm (Online Mirror Descent with Tsallis regularizer) of~\cite{zimmert2021tsallis}. Of course, as for the classical adversarial multi-armed bandits, the losses $\ell_{i,t}(k) = o_t(k_{-i,t},k)$ cannot be observed for all $k \in [K]$. Therefore, they are estimated in the algorithm with the importance weight estimators
\begin{equation}
	\label{eq:hat_g}
	\hat{\ell}_{i, t}(k) = \begin{cases}
		\ell_{i, t}(k) / p_{i, t}(k) & \text{if } k = k_{i, t} \\
		0 & \text{otherwise}
	\end{cases} \,.
\end{equation}
The resulted algorithm is described in Algorithm~\ref{alg:bdb}.

\begin{algorithm}[h]
   \caption{\textbf{\algbdb\, (Best-of-Both DB)}}
   \label{alg:bdb}
\begin{algorithmic}[1]	
   \STATE {\bfseries input:} Arm set: $[K]$, {$(\Psi_t)_{t=1,2,\dots}$}
   \STATE {\bfseries init:} $G_{i,0} \leftarrow \0_K$ for $i \in \{+1, -1\}$
    \FOR{$t = 1, 2, \ldots, T$}
		\STATE choose $\p_{i,t} = \nabla(\Psi_t+\mathcal{I}_{\Delta})^*(-\hat L_{i,t-1})$ 
		\STATE For $i \in \{+1, -1\}$, sample $k_{i, t}$ from the distribution $(p_{i, t}(1), \dots, p_{i, t}(K))$
		\STATE Observe preference feedback $o_t(k_{+1, t}, k_{-1, t})$
		\STATE Compute $\hat{\ell}_{i, t}(k)$ for $i \in \{+1, -1\}$ and $k \in [K]$ using \eqref{eq:hat_g}
		\STATE update $\hat L_{i,t} = \hat L_{i,t-1}+\hat{\ell}_{i,t}$\;	
   \ENDFOR  
\end{algorithmic} 
\end{algorithm} 

\begin{restatable}[]{thm}{mainthm}
\label{thm:mainthm}
For any sequence of preference matrices $P_t$, the pseudo-regret of Algorithm~\ref{alg:bdb} with $\Psi_t(w) = \sqrt{t} \sum_{k=1}^K (\sqrt{w_k} - w_k/2)/8$ satisfies for any $T\geq 1$
\[
   \overline{R}_T := \max_{k \in [K]} \E\big[R_T(k)] \leq 4 \sqrt{KT} + 1.
\]
Furthermore, if there exists a gap vector $\Delta \in [0,1]^K$ with a unique zero coordinate $k^* \in [K]$ and $C \geq 0$ such that 
\begin{equation}
  \label{eq:self_bounded}
  \overline{R}_T \geq \frac{1}{2} \E\bigg[ \sum_{t=1}^T  \sum_{k \neq k^*} \big(p_{+1,t}(k) + p_{-1,t}(k)\big) \Delta_k \bigg] - C\,,
\end{equation}
the pseudo regret also satisfies
\begin{align*}
  \overline{R}_T \leq \sum_{k\neq k^*} \frac{4 \log T + 12}{\Delta_k} + 4 \log T + \frac{1}{\Delta_{\min}} + \frac{3}{2} \sqrt{K} + 8 + C, 
\end{align*}
where $\Delta_{\min} = \min_{k \neq k^*} \Delta_k$. 
\end{restatable}

The proof is postponed to Appendix~\ref{sec:reg_analysis}. Note that the theorem largely follows from (and is itself highly similar to)  the best-of-both worlds regret-bound  of \citep[Theorem~1]{zimmert2021tsallis} for MAB.  We insist on the fact that our contribution should not be seen as technical.  Indeed, the proof  is just a clever combination of their MAB analysis with our black box reduction (Theorem~\ref{thm:reduction}). But we believe that the simplicity of our approach is its strength that can benefit the community of dueling bandits. As we will see, several state-of-the-art results of dueling bandits can be simultaneously improved as a direct consequence of this theorem.

Note that for simplicity, we restricted ourselves to importance weighted estimators~\eqref{eq:hat_g}. By using more sophisticated variance reduced estimators, as in \cite{zimmert2021tsallis}, the multiplicative constants can be reduced. Furthermore, similar to \cite{zimmert2021tsallis}, the result holds only for the pseudo-regret and not for the true regret. \citet{auer2016algorithm} have indeed proven that no optimal adversarial and stochastic high probability regret bounds can be obtained simultaneously for standard stochastic bandits. The learner must pay suboptimal logarithmic factors. The result can be extended to dueling bandits.

It is worth to emphasize that this is the first best-of-both worlds regret bound for general dueling bandits (the stochastic bound follows from the choice $C=0$, see Sec~\ref{sec:stoch}). \citet[Cor.~10]{zimmert2021tsallis} obtain a similar result for the very same algorithm but for utility based dueling bandits only.

\begin{rem}
Note that a single sub-routine of OMD to optimize the weights is actually enough to get the same regret guarantee. To do so, one samples both $k_{-1,t}$ and $k_{+1,t}$ independently from the same distribution $p_t = \nabla (\psi_t + \cI_\Delta)^*(-\hat L_{t-1})$. Here, $\hat L_t = \sum_{s=1}^t \hat \ell_s \in \R_+^K$ and the importance weight estimator are defined for all $k \in [K]$ by
\[
  \hat \ell_t(k) = \frac{1}{2}\big( \hat \ell_{-1,t}(k) + \hat \ell_{+1,t}(k)\big) \,.
\]
Noting that $\E[\hat \ell_t(k)] = \E\big[P_t(k_{-1,t},k) + P_t(k_{+1,t},k)\big]/2$ and $\E\big[\sum_{k=1}^Kp_t(k)  \hat \ell_t(k)\big] = \E\big[P_t(k_{-1,t},k_{+1,t}) + P_t(k_{+1,t},k_{-1,t})\big]/2 = 1/2$, the proof follows similarly to the other one. Though the regret upper-bound is exactly the same, we believe that this version might lead to better performance because the two players share information.
\end{rem}


%% file: corrupted_regime.tex

\section{Improvements Over Existing Dueling Bandit Results}

Also our approach and analysis is rather simple, it allows to outperform the best existing regret-upper bounds for stochastic dueling bandits with or without corruption. We believe that the dueling bandit community will benefit from this reduction and that it may be applied to a wider scope such as to deal with non-stationarity, delays, or non-standard feedbacks (graphs) for which many results already exist in standard multi-armed bandits.

\subsection{Stochastic Dueling Bandits with Condorcet Winner} 
\label{sec:stoch}
In stochastic dueling bandits, the preference matrices $P_t$ are fixed over time $P_t = P$ for all $t \geq 1$. Under the Condorcet winner assumption there exists $k^{\texttt{(cw)}} \in [K]$ such that $P(k^{\texttt{(cw)}},k) > \frac{1}{2}$ for all $k \neq k^{\texttt{(cw)}}$. Then, the suboptimality gaps of all actions $k \in [K]$ are defined as
$
    \Delta_k  := P(k^{\texttt{(cw)}}, k) - \frac{1}{2}.
$
Remarking that in this case the self-bounding assumption~\eqref{eq:self_bounded} is satisfied with $C = 0$, since
\begin{align*}
  \overline{R}_T 
    & = \frac{1}{2} \sum_{t=1}^T \E\Big[ P_t(k^{\texttt{(cw)}},k_{+1,t}) + P_t(k^{\texttt{(cw)}},k_{-1,t}) - 1\Big] \\
    & = \frac{1}{2} \sum_{t=1}^T \E\Big[\Delta_{k_{+1,t}} + \Delta_{k_{-1,t}}\Big] \\
    & = \frac{1}{2} \sum_{t=1}^T \E\Big[ \sum_{k \neq k^{\texttt{(cw)}}} (p_{-1,t} + p_{+1,t}) \Delta_k \Big]\,,
\end{align*}
we get the following corollary from Theorem~\ref{thm:mainthm}. 

\begin{corollary} For stochastic dueling bandits with Condorcet winner, the pseudo-regret of Algorithm~\ref{alg:bdb} with well-chosen parameters satisfies
\[
  \overline{R}_T \leq \sum_{k\neq k^{\texttt{(cw)}}} \frac{4 \log T + 12}{\Delta_k} + 4 \log T \\
    + \frac{1}{\Delta_{\min}} + \frac{3}{2} \sqrt{K} + 8\,.
\]
\end{corollary}

Note that the above bound is the first pseudo-regret upper-bound for stochastic dueling bandit that does not suffer from a $\Delta_{\min}^{-2}$ dependence under the Condorcet winner assumption, without a quadratic dependence on the number of arms, $K$, which is a concern when it comes to dealing with
large-scale problems. For instance, RUCB \cite{Zoghi+14RUCB} satisfies a regret bound of order $O(K \log(T)  \Delta_{\min}^{-2} + K^2)$, MergeRUCB \cite{Zoghi+15MRUCB} has linear dependence on $K$ but suffers $O(K \log(T)\Delta_{\min}^{-2})$. Finally, RMED from \cite{Komiyama+15} is asymptotically optimal when $T\to \infty$ but also suffers from large constant terms ($K^2$ and $\Delta_{\min}^{-2}$) and is only valid for $K\to \infty$. We refer the reader to \cite{bengs2021preference} for existing results on stochastic dueling bandits.

\subsection{Corrupted Regime}
\label{sec:corruption}

Here, we consider the stochastic dueling bandit problem in the presence of adversarial corruptions. The robustness to adversarial corruption has known recent progress in the MAB setting \citep{gupta2019better,lykouris2018stochastic,zimmert2021tsallis}  and has recently been extended to the DB framework \citep{agarwal2021stochastic}. 
The preference matrices are fixed $P_t = P$ for all $t\geq 1$ and we assume the existence of a Condorcet winner $k^{\texttt{(cw)}}$. Furthermore, an adversary may corrupt the outcomes of some duels by replacing the results of the duels $o_t(k,k')$ with corrupted ones $\tilde o_t(k,k')$. At the end of each round, the player only observes $\tilde o_t(k_{+1,t},k_{-1,t})$. The objective of the player is to minimize the pseudo-regret $\bar R_T$ under a bounded total amount of corruption 
\[
     C := \sum_{t=1}^T \sum_{k \neq k^{\texttt{(cw)}}} \big|o_t(k^{\texttt{(cw)}},k) - \tilde o_t(k^{\texttt{(cw)}},k)\big|.
\]
We show here that similarly to what happens for standard Multi-armed bandits in \cite{zimmert2021tsallis}, this corrupted setting is a special case of the self-bounding assumption~\eqref{eq:self_bounded}. Indeed, defining $\tilde P_t$ the corrupted preference matrices by $\tilde P_t(k,k') = \E\big[\tilde o_t(k,k')\big]$ and the corrupted pseudo-regret 
\[
    \widetilde{R}_T(k) = \frac{1}{2} \E\bigg[\sum_{t=1}^T \tilde P_t(k,k_{+1,t}) + P_t(k,k_{-1,t}) -1\bigg] \,,
\] 
we have
\begin{align*}
     & \widetilde{R}_T(k^{\texttt{(cw)}}) \\
        &= \frac{1}{2} \E\bigg[ \sum_{t=1}^T \tilde o_t(k^{\texttt{(cw)}},k_{+1,t}) + \tilde o_t(k^{\texttt{(cw)}},k_{-1,t}) - 1\bigg] \\
        &\geq  \frac{1}{2} \E\bigg[ \sum_{t=1}^T o_t(k^{\texttt{(cw)}},k_{+1,t}) + o_t(k^{\texttt{(cw)}},k_{-1,t}) - 1\bigg] -  C \\
        & = \overline{R}_T -  C \\
        & = \frac{1}{2} \sum_{t=1}^T \E\Big[ \sum_{k \neq k^{\texttt{(cw)}}} (p_{-1,t} + p_{+1,t}) \Delta_k \Big] -  C \,.
\end{align*}
Therefore, the corrupted regime satisfies the self-bounding assumption~\eqref{eq:self_bounded}. Applying Theorem~\ref{thm:mainthm} on the corrupted regime and using that we also have $\overline{R}_T(k^{\texttt{(cw)}}) \leq \widetilde{R}_T(k^{\texttt{(cw)}}) + C$, we get the following corollary.

\begin{corollary}\label{cor:corruption}
For stochastic dueling bandits with Condorcet winner and corruptions, whose total amount is bounded by $C$, the pseudo-regret $\overline{R}_T$ Algorithm~\ref{alg:bdb} is upper-bounded as
\begin{align*}
    \overline{R}_T \leq \sum_{k\neq k^{\texttt{(cw)}}} \frac{4 \log T + 12}{\Delta_k} + 4 \log T 
    + \frac{1}{\Delta_{\min}} + \frac{3}{2} \sqrt{K} + 8 + 2C.
\end{align*}
\end{corollary}

Although  Corollary~\ref{cor:corruption} easily follows from Theorem~\ref{thm:mainthm} which itself easily follows from Theorem~1 of  \cite{zimmert2021tsallis}, the latter result significantly improves upon the recent results on dueling bandit with corruptions obtained by \cite{agarwal2021stochastic}. Indeed, the latter provide for the same setting and a significantly more sophisticated procedure a high-probability regret bound of order
\[
    O\left( \frac{K^2 C}{\Delta_{\min}} + \sum_{k \neq k^{\texttt{(cw)}}} \frac{K^2}{\Delta_k^2} \log \Big(\frac{K}{\Delta_k}\Big) + \sum_{k \neq k^{\texttt{(cw)}}} \frac{\log T}{\Delta_k} \right) \,.
\]
As often in the dueling bandit literature, it suffers from both a quadratic dependence on the number of actions and $\Delta_{\min}^{-1}$. Furthermore, our regret bound is sublinear in $T$ as soon as the corruption level is $o(T)$ while \citet{agarwal2021stochastic} can only afford $o(\Delta_{\min} T/K^2)$. 

Moreover, \citet{zimmert2021tsallis} also provide an upper-bound for stochastic bandits with adversarial corruption. The latter is of order 
$
    O\bigg(\sum_{k\neq k^{\texttt{(cw)}}} \frac{\log T}{\Delta_k}  + \sqrt{C \sum_{k\neq k^{\texttt{(cw)}}} \frac{\log T}{\Delta_k} }\bigg)\,,
$
which seems to outperform our bound when $C$ is large. The difference is since they upper-bound the corrupted regret $\widetilde R_T(k^{\texttt{(cw)}})$ and not $\bar R_T$.

%% file: expts_arxiv.tex
\section{Experiments}
\label{sec:expts}


\textbf{Algorithms. } 
We compared the performances of the following algorithms: 
$1.$ VDB: Our proposed \algbdb\, (Alg. \ref{alg:bdb}) algorithm. 
$2.$ RUCB: The algorithm proposed in \cite{Zoghi+14RUCB} for K-armed stochastic dueling bandits for CW regret. (We set the algorithm parameter $\alpha = 0.6$).
$3.$ RMED: Another algorithm for CW regret as proposed in \cite{Komiyama+15}. (We set the algorithm parameter $f(K) = 0.3K^{1.01}$ as suggested in their experimental evaluation).
$4.$ DTS: The double thompson sampling algorithm of \cite{DTS}. (Here again we set the similar algorithm parameter $\alpha = 0.6$).
$5.$ REX3: As introduced in \cite{Adv_DB}. Note that their suggested optimal tuning parameters, i.e. the uniform exploration rate $\gamma$ as well as the learning rate $\eta$ requires the knowledge of problem dependent parameters $\tau$ (see Thm. 1 of \cite{Adv_DB}) which are unknown to the learner. We used $T$ in place of $\tau$ henceforth. 

\textbf{Performance Measures.} We report the average cumulative regret (Eqn. \eqref{eq:sreg}) of the algorithms averaged over $20$ runs.

\subsection{Stochastic Preferences}
We compared their regret performance across the following stochastic environments:
 
\textbf{Constructing Preference Matrices ($\P$).} 
We use four different utility parameter $\btheta = (\theta_1,\ldots,\theta_K)$ based preference models where the underlying preference model is defined as $P(i,j):= \frac{\theta_i}{\theta_i + \theta_j} ~\forall i,j \in [K]$. 
The model is famously studied as BTL model or model generally Plackett-Luce model \cite{SGwin18,ChenSoda+18,shah17}. Note this ensures $P$ to have \emph{total-ordering} \cite{BTM,falahatgar_nips}.

In particular we consider the following choices of $\btheta$:
$1.$ {\it Trivial} $2.$ {\it Easy} $3.$ {\it Medium}, and $4.$ {\it Hard} with their respective $\btheta$ parameters are given by:
$1.$ {\it Trivial:} $\btheta(1) = 1$, $\btheta(2:K) = 0.5$.
$2.$ {\it Easy:} $\btheta(1:\lfloor K /2 \rfloor) = 1$, $\btheta(\lfloor K/2 \rfloor + 1:K) = 0.5$.
$3.$ {\it Medium:} $\btheta(1:\lfloor K/3 \rfloor) = 1$, $\btheta(\lfloor K/3 \rfloor+1:\lfloor 2K/3 \rfloor) = 0.7$, $\btheta(\lfloor 2K/3 \rfloor + 1:K) = 0.4$.
$4.$ {\it Hard:} $\btheta(i) = 1 - (i-1)/K,\, \forall i \in [K]$. Note for each $\bsigma^* = (1 > 2> \ldots K)$. 
For the purpose of our experiments we set $K=10$. 
We also evaluated the algorithms on two general $10\times 10$ preference matrices {\it Car} and {\it Hurdy} as also used in \cite{niranjan2017,SG18}.

\textbf{Regret vs Time.}
Fig. \ref{fig:vs_t} shows the relative performances of different algorithms with time. As follows from the plots, in general \textit{VDB} (\algbdb) outperforms the rest in all instances, with \textit{DTS} being closely competitive in some cases. In terms of the problem hardness, as their names suggest too, the
\textit{Trivial} and \textit{Easy} instances are easiest to learn as the best-vs-worst item preferences are well separated in these cases and the diversity of the item preferences across different groups are least. Consequently the algorithms yield slightly more regret on \emph{instance-Medium} due to higher preference diversity, and the hardest instance being \emph{Hard} where the algorithms require maximum time to converge, though \textit{VDB} reaching the convergence fastest still. 

\begin{figure}[h]
	\begin{center}
		\includegraphics[width=0.33\textwidth]{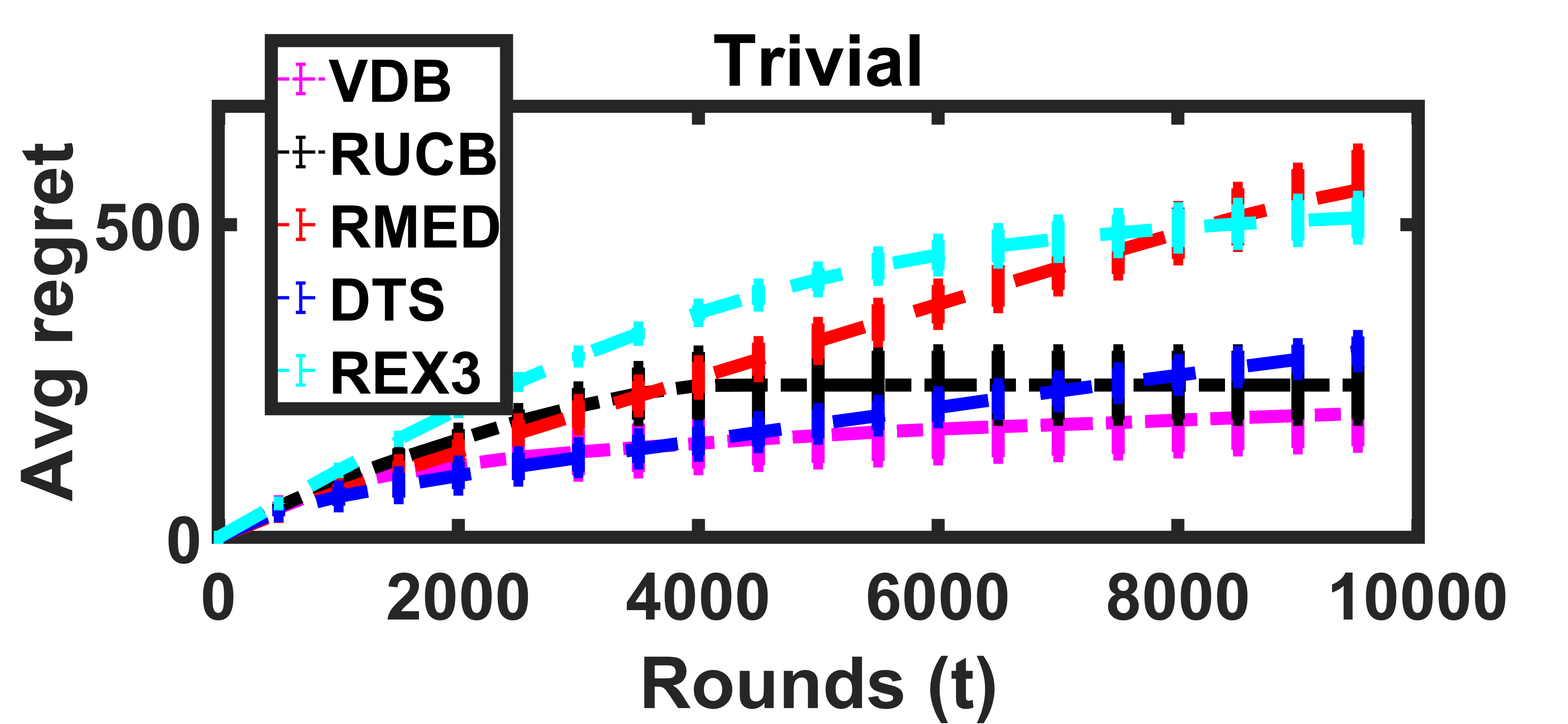}\hspace{-5pt}
		\includegraphics[width=0.33\textwidth]{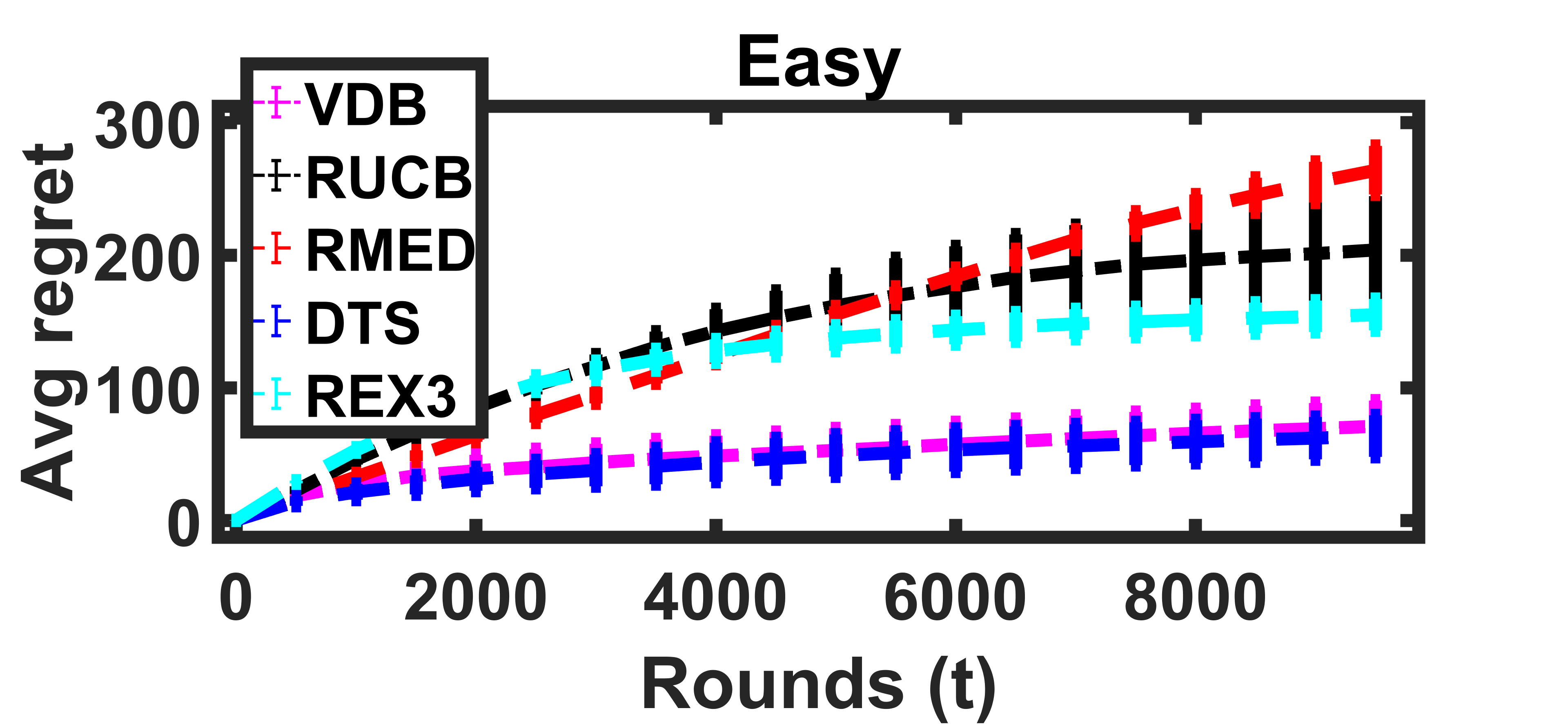}
		\includegraphics[width=0.33\textwidth]{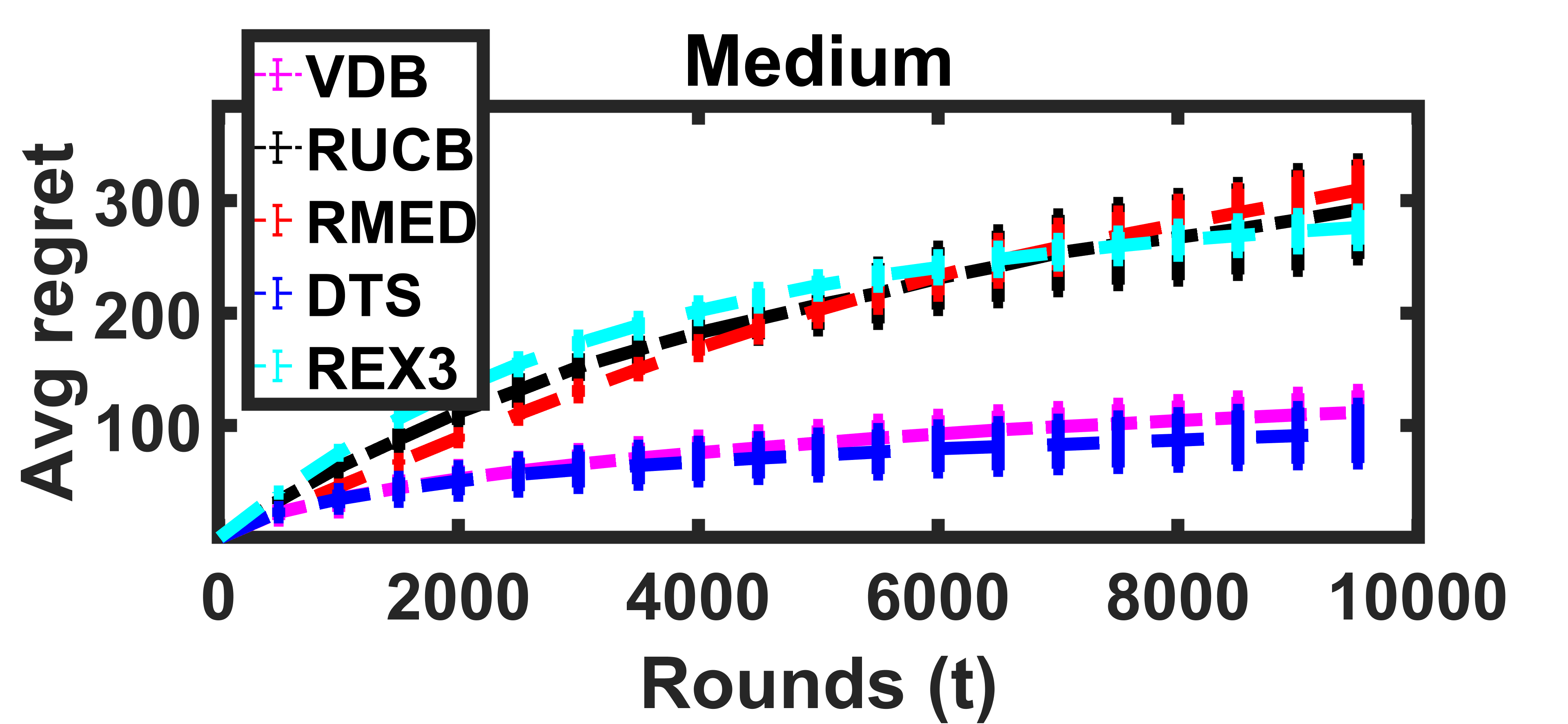}\hspace{-5pt}
		\includegraphics[width=0.33\textwidth]{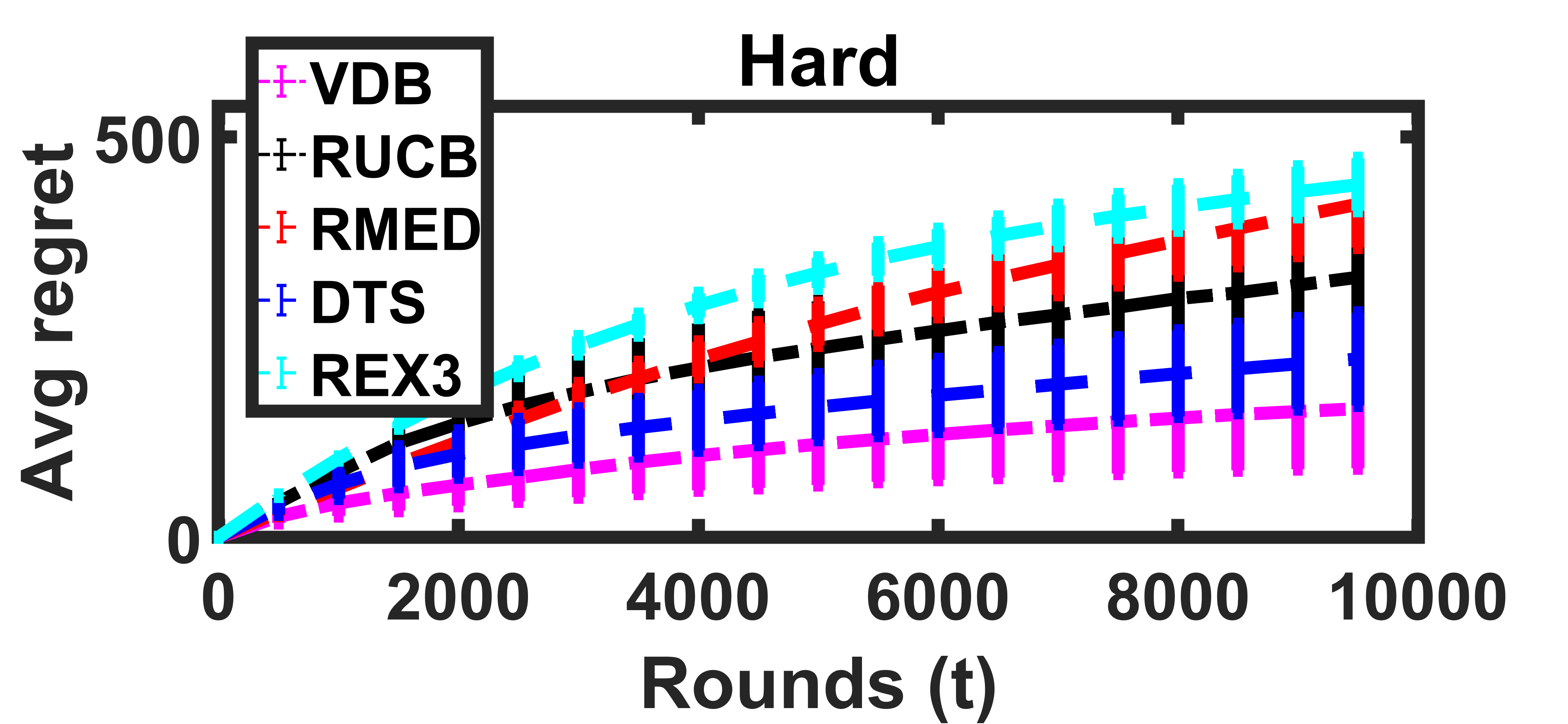}
		\includegraphics[width=0.33\textwidth]{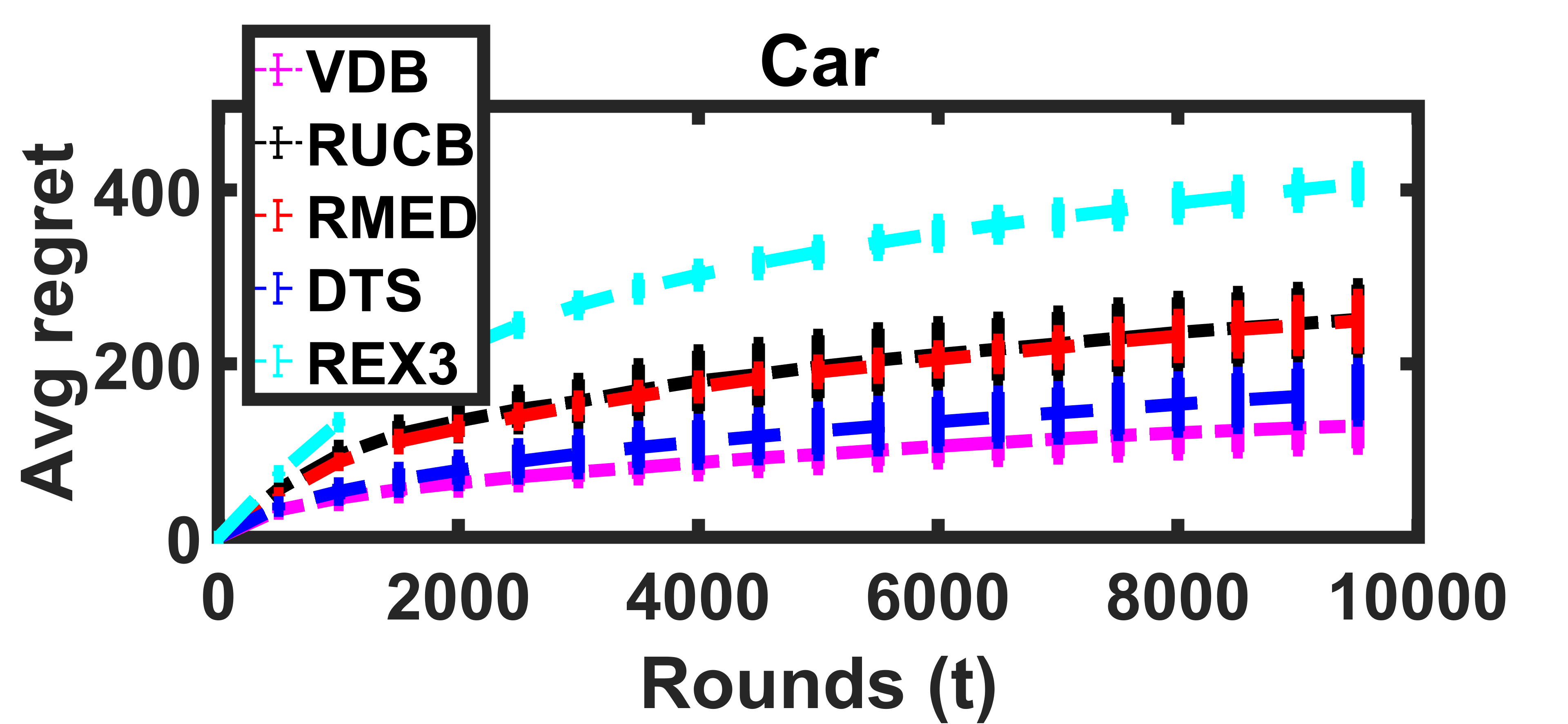}\hspace{-5pt}
		\includegraphics[width=0.33\textwidth]{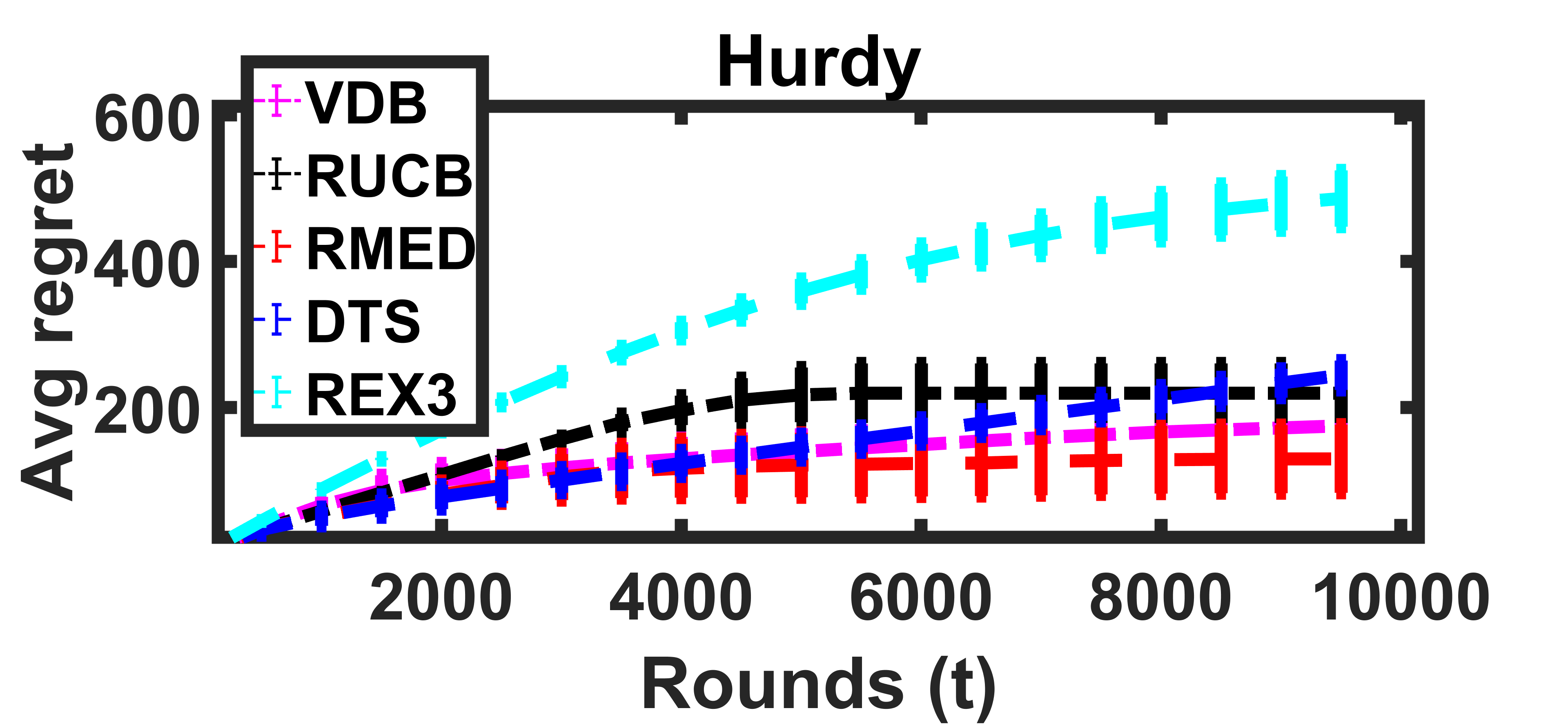}
		\vspace{-10pt}
		\caption{Averaged cumulative regret over time}
		\label{fig:vs_t}
	\end{center}
\end{figure}
\smallskip 

\subsection{Corrupted Preferences}
We also evaluated the performances of algorithms in presence of corruption (Sec. \ref{sec:corruption}). In particular, Fig. \ref{fig:vs_corr2} and \ref{fig:vs_corr4} respectively shows the relative performances of the algorithms with $20\%$ and $40\%$ corrupted feedback (at each round, we flip the winner feedback with that certain probability) respectively on \textit{Medium} and \textit{Hard} Plackett-Luce instances. As expected, the performances of all the algorithms decay significantly with increasing degree of feedback-corruption, however as before, \textit{VDB} consistently performed best over all the baselines and tend to converge the fastest among all.

\begin{figure}[h]
	\begin{center}
		\includegraphics[width=0.34\textwidth]{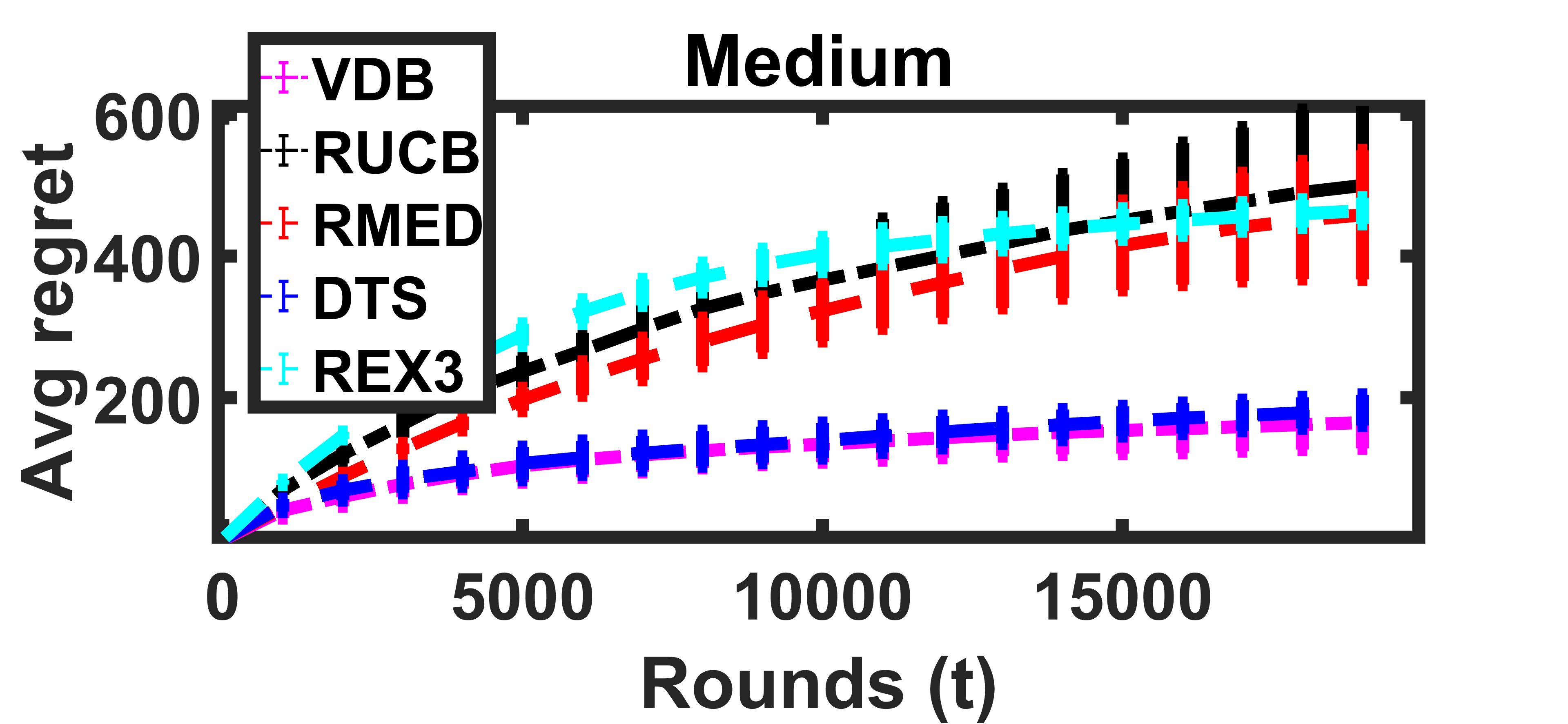}\hspace{-5pt}
		\includegraphics[width=0.34\textwidth]{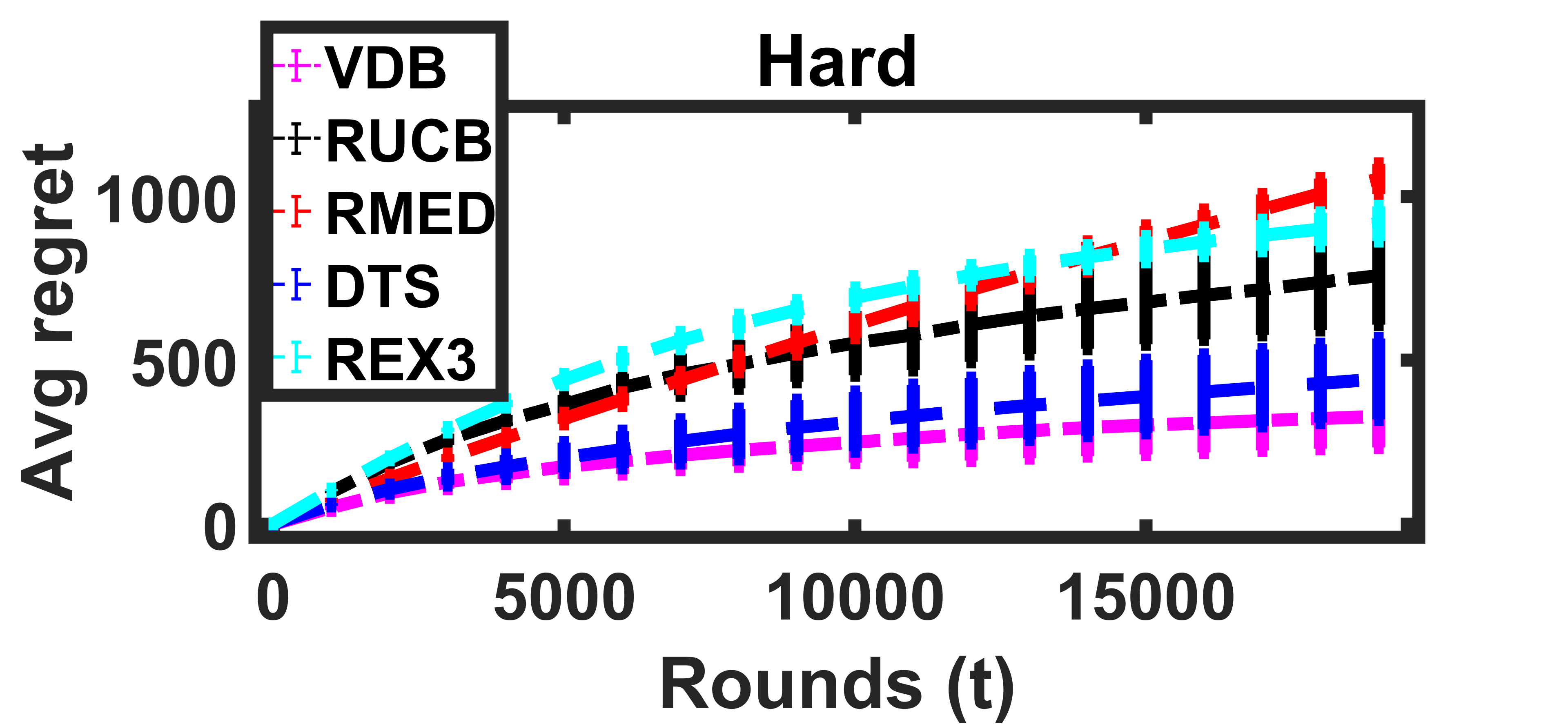}
		\vspace{-10pt}
		\caption{Averaged cumulative regret ($20\%$ corrupted feedback)}
		\label{fig:vs_corr2}
	\end{center}
\end{figure}
\vspace{-20pt}
\begin{figure}[h]
	\begin{center}
		\includegraphics[width=0.34\textwidth]{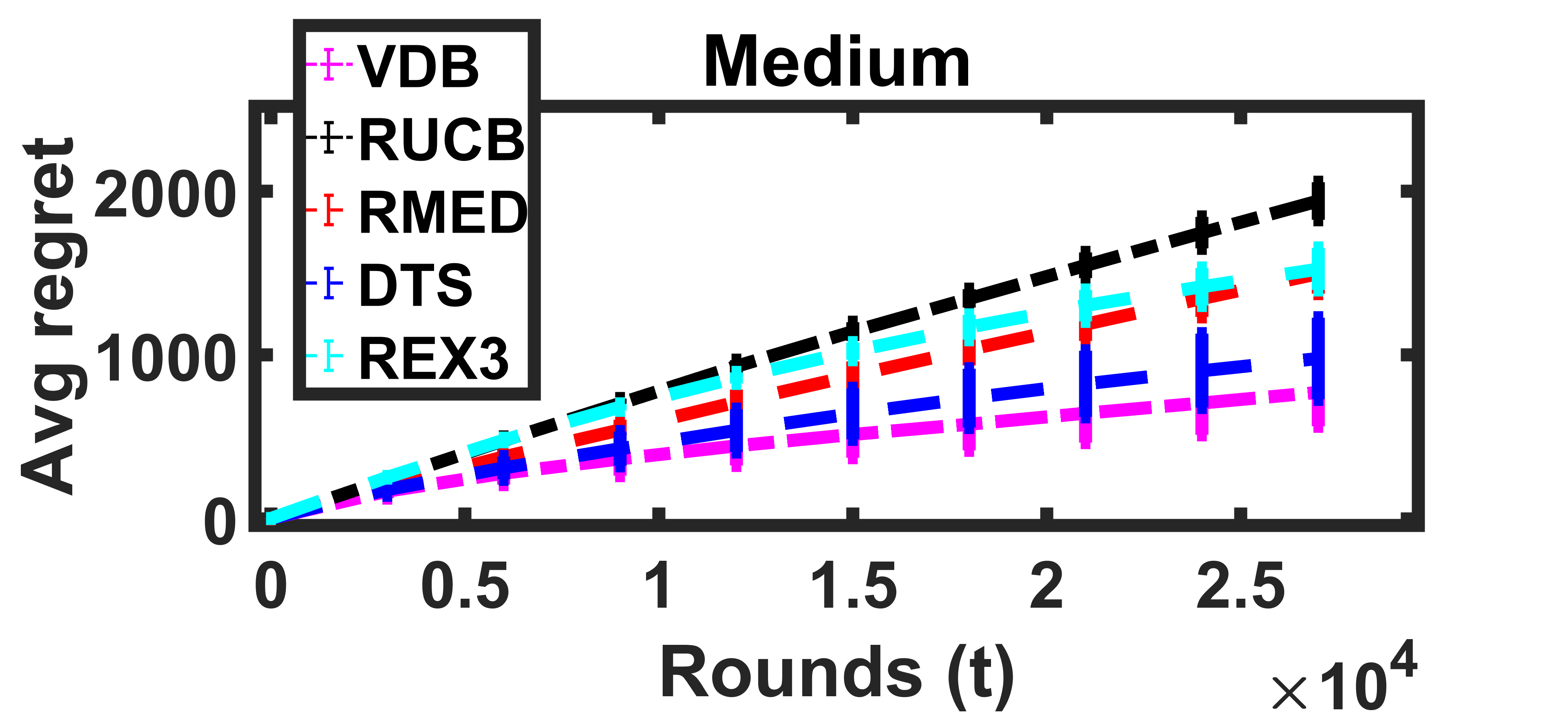}\hspace{-5pt}
		\includegraphics[width=0.34\textwidth]{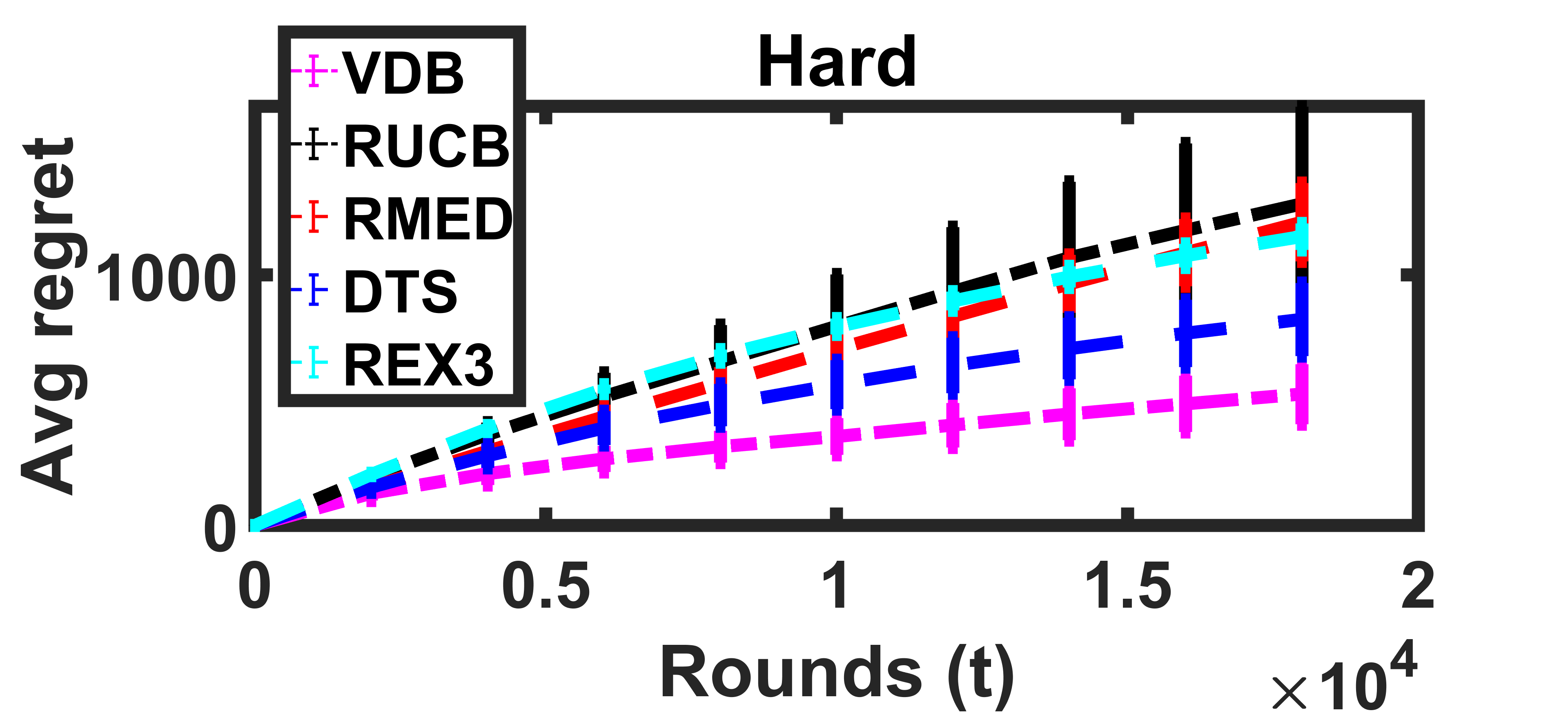}
		\vspace{-10pt}
		\caption{Averaged cumulative regret ($40\%$ corrupted feedback)}
		\label{fig:vs_corr4}
	\end{center}
\end{figure}

%% file: conclusion.tex
\vspace{-10pt}
\section{Discussions}
\label{sec:concl}

We studied the problem of \bdb, which gives the first \emph{`best-of-both' world} result for the problem of Dueling Bandits. The crux of our analyses relies on a \emph{novel idea of decomposing the dueling bandit regret into multiarmed bandit (MAB) regret} by interpreting the dueling preference feedback as a certain realization of adversarial reward sequence. 
An important byproduct of our best-of-both dueling analysis is, this gives the first order optimal gap-dependent regret bound for $K$-armed stochastic dueling bandits, closing the decade-long open problem of tightness of `Condorcet dueling bandit regret'.
Further we also analyze the robustness of our algorithm under corrupted preference feedback setting, which provably improves over the state-of-the art corrupted dueling bandits algorithms.

\textbf{Future Works.} 
Proving the first \emph{best-of-both world} result for dueling bandits using our novel reduction idea is just a first step towards exploring the possibility of understanding how far this idea can be extended to apply existing multiarmed bandits results to dueling bandits frameworks, 
instead of putting individual and isolated efforts in developing dueling bandit algorithms, taking inspirations from existing MAB generalizations. 
Some such extensions could be to analyze dynamic dueling bandit regret under non-stationary preferences \cite{wei21,luo+19,besbes+15}, item non-availability \cite{neu14,kanade09}, delayed feedback \cite{delay1,delay2,delay3}, budget constraints \cite{budget1,budget2,budget3}, or even the more general reinforcement learning (RL) scenarios \cite{ucrl,talebi18,ng06}, for which there are already well established theory of works with MAB framework. 
Also under what settings of Dueling Bandits, its corresponding MAB counterpart based reductions are bound to fail?
%
 
Finally it is worth mentioning that, an ambitious (and broad) objective along these line of thoughts is to understand the connection between different learning scenarios to dueling bandits, 
e.g. feedback graphs \cite{Alon+15,Alon+17},
partial monitoring problems \cite{pm1,pm2,pm3},
markov games \cite{xie+20,bai+20,bai20},
etc. The obvious motivation being to understand how far we can re-engineer the existing results from related learning literature to solve the online preference bandits problems.

%% file: appendix.tex

\section*{\centering\large{Supplementary for \papertitle}}
\vspace*{1cm}
	
\section{Regret analysis of Algorithm~\ref{alg:rr}}

\label{app:proofs}
\thmrr*

\begin{proof}[Proof of Theorem~\ref{thm:rr}]
Let us denote by 
\[
u_{ij}(t) := \hp_{ij}(t)+c_{ij}(t)
\]
for any pair $(i,j)$ and time $t$, where 
\[
    c_{ij}(t):=\sqrt{\frac{ \log (Kt/\delta)}{n_{ij}(t)}} \,,
\]
and assume $\Delta_2 \leq \Delta_3 \leq \cdots \leq \Delta_K$ without loss of generality.
We will also assume the confidence bounds of Lem. \ref{lem:conf} holds good for all $t \in [T]$ and all pairs $(i,j)$, which is shown to hold good with probability at least $1-\delta$. In particular, this implies that the best arm cannot be eliminated, i.e., $1\in \cA_t$ for all $t\geq 1$. 

We start by noting that if the worst arm $K$ (since $\Delta_K = \max_{i = 2}^{K}\Delta_i$, arm-K gets maximally beaten by the CW) is played at time $t$, it means $u_{K1}(t) \geq \frac{1}{2}$.
However we also have,
\begin{align*}
u_{K1}(t) &= \hp_{K1}(t) + c_{K1}(t)\\
 & \leq p_{K1} + 2c_{K1}(t) = 1/2 - \Delta_K + 2c_{K1}(t),
\end{align*}
where the inequality holds due to Lem. \ref{lem:conf} and the last equality holds by noting $p_{K1} = 1-p_{1K} = 1 - (1/2 + \Delta_K) = 1/2 - \Delta_K$.

So $u_{K1}(t) > 1/2$ can only hold good if $c_{K1}(t) > \Delta_K/2$ which implies, 
\begin{align}
	\label{eq:rr1}
& n_{K1}(t) \leq \frac{4 \log (Kt/\delta)}{\Delta_K^2}.
\end{align}

But by the our algorithm design since all the pairs are drawn in a round robin fashion, at any round $t \in [T]$, for any two distinct pairs $(i,j)$ and $(i',j')$ that are in $\cA_t$ note that 
\begin{align}
	\label{eq:rr0}
|n_{ij}(t) - n_{i'j'}(t)| \leq 1\,.
\end{align}
Thus the total regret incurred by Alg. \ref{alg:rr} at rounds where $K \in \{k_{+1,t},k_{-1,t}\}$, can be upper bounded as: 
\begin{align*}
\sum_{t=1}^T\sum_{k < K} & \1(\{k_{+1,t},k_{-1,t}\} = \{k,K\}) \frac{\Delta_{K} + \Delta_k}{2} \\
& \leq \sum_{k=1}^{K-1} n_{kK}(T) \Delta_K \leq (K-1) (1+n_{K1}(T)) \Delta_K \\
& \leq  (K-1) \Big(1+ \frac{4 \log (Kt/\delta)}{\Delta_K^2}\Big) \Delta_K \\ 
& = (K-1) \Big(\Delta_k + \frac{4 \log (Kt/\delta)}{\Delta_K}\Big) \,.
\end{align*}

Similarly, note for any $i \in \{2,3,\ldots K-1\}$, we can upper bound the regret of rounds where $i$ was played in the duel as:
\begin{align*}
	\sum_{t=1}^T\sum_{k < i} & \1(\{k_{+1,t},k_{-1,t}\} = \{k,i\}) \frac{\Delta_{i} + \Delta_k}{2} \\
	& \leq \sum_{k=1}^{i-1} n_{ki}(T) \Delta_i \leq (i-1) (1+n_{1i}(T)) \Delta_i \\
	& \leq (i-1) \Big(\Delta_i + \frac{4 \log (Kt/\delta)}{\Delta_i}\Big)\,.
\end{align*}

Thus we can bound the total regret of Algorithm \ref{alg:rr} as:

\begin{align*}
	\sum_{t=1}^T\sum_{i = 2}^K \sum_{k = 1}^{i-1} & \1(\{k_{+1,t},k_{-1,t}\} = \{k,i\}) \frac{\Delta_{i} + \Delta_k}{2} \\
	& \leq \sum_{i = 2}^K  (i-1) \Big(\Delta_i + \frac{4 \log (Kt/\delta)}{\Delta_i}\Big) \\
	& \stackrel{(\Delta_i\leq 1/2)}{\leq} \frac{K^2}{4} + 4 \sum_{i = 2}^K  (i-1)  \frac{ \log (Kt/\delta) }{\Delta_i}
\end{align*}
which concludes the first half of the proof.
Further, to show the second part of the claim (analyzing worst-case gap-independent regret bound of Algorithm \ref{alg:rr}), note that Eqn. \eqref{eq:rr1} equivalently implies for any $i \in [K]\sm \{1\}$:
\begin{align*}
\Delta_i \leq \sqrt{\frac{4 \log (Kt/\delta)}{n_{i1}(t)}} \,.
\end{align*}
Hence we can alternatively upper bound the regret as:
\begin{align*}
R_T &=  \sum_{i=2}^K \sum_{k=1}^{i-1} n_{ik}(T) \frac{\Delta_j + \Delta_k}{2} \leq \sum_{i=2}^K \sum_{k=1}^{i-1} n_{ik}(T) \Delta_i\\
& \leq \sum_{i=2}^K \sum_{k=1}^{i-1} n_{ik}(T) \sqrt{\frac{4 \log (KT/\delta)}{n_{i1}(T)}} \\ 
& \overset{(a)}{\leq} 2 \sum_{i=2}^K \sum_{k=1}^{i-1} \sqrt{2 n_{ik}(T) \log (Kt/\delta)} \\
& \overset{(b)}{\leq} \sum_{i=2}^{K} 2\sqrt{K^2  \sum_{i=2}^K \sum_{k=1}^{i-1} n_{ik}(T)\log(KT/\delta) }\\
&\leq 2K \sqrt{ T \log(KT/\delta)}\,,
\end{align*}
where $(a)$ follows from the observation of Eqn.~\eqref{eq:rr0} which implies $n_{i1}(T) \geq n_{ik}(T)$ when $T \geq K^2$ and $(b)$ from Jensen's inequality and $\sum_{i=2}^K(i-1) \leq K^2/2$.
\end{proof}

\begin{restatable}[]{lem}{lemconf}
	\label{lem:conf}
	For any $\delta \in (0,1/2)$. Then, with probability at least $1-\delta$, for any pair $i,j \in [K]$ and any time $t \in [T]$%
	\[
	\hp_{ij}(t)-c_{ij}(t) \leq p_{ij} \leq \hp_{ij}(t)+c_{ij}(t),  \qquad \forall t \in [T] ,
	\]
	where $c_{ij}(t):=\sqrt{\frac{\log (Kt/\delta)}{n_{ij}(t)}}$.
\end{restatable}

\begin{proof}
Let us denote by $u_{ij}(t) := \hp_{ij}(t)+c_{ij}(t)$ and $ \ell_{ij}(t) := \hp_{ij}(t)+c_{ij}(t)$.
	Note the inequality holds trivially at round $t$, for any pair $(i,j)$ for which $n_{ij}(t) = 0$ since in these cases $\ell_{ij}(t) \leq 0$ and $u_{ij}(t) \geq 1$.
	
	Now consider any pair $(i,j)$ and round $t \in [T]$ such that $n_{ij}(t) > 0$. Note in this case by Hoeffding's Inequality: 
	\begin{multline*}
	Pr  \Bigg(|p_{ij}-\hat p_{ij}(t) | > \sqrt{\frac{\ln (Kt/\delta)}{n_{ij}(t)}} \Bigg) \\
	 \leq 2e^{-2n_{ij}(t)\frac{\ln (Kt/\delta)}{n_{ij}(t)}} = \frac{2\delta^2}{K^2t^2} \leq \frac{\delta}{K^2t^2}\,.
	\end{multline*}
	Taking union bound over all $K \choose 2$ pairs and time $t \in [T]$ we get:
	\begin{align*}
		Pr & \Bigg(\exists i,j \in [K], t \in [T] \text{ s.t. } |p_{ij}-\hat p_{ij}(t) | > \sqrt{\frac{\ln (Kt/\delta)}{n_{ij}(t)}} \Bigg) \\
		& \leq \sum_{t = 1}^T \sum_{i=2}^{K}\sum_{j = 1}^{i} \frac{\delta}{K^2t^2} \leq \sum_{t = 1}^\infty \frac{\delta}{2t^2} \leq \frac{\delta\pi^2}{12} \leq \delta,
	\end{align*}
	 where in the second last inequality we used $\sum_{t = 1}^\infty \frac{1}{t^2} < \frac{\pi^2}{6}$. This concludes the claim.
\end{proof}

%% file: algo_analysis.tex

\section{Regret Analysis of Alg. \ref{alg:bdb}}
\label{sec:reg_analysis}

\mainthm*

\begin{proof}[Proof of Theorem~\ref{thm:mainthm}]


The analysis follows from carefully combining our reduction (Theorem~\ref{thm:reduction}) with Theorem~1 of~\cite{zimmert2021tsallis} for MAB to both of the players. Indeed, for each player $i\in \{-1,1\}$, Algorithm~\ref{alg:bdb} chooses $k_{i,t}$ by following the decisions of Tsallis-INF \citep[Alg.~1]{zimmert2021tsallis} with $\alpha = 1/2$, symmetric regularization, learning rate $\eta_t = 4/\sqrt{t}$ and losses $\ell_{i,t}$ estimated in~\eqref{eq:hat_g} with standard importance sampling (IW).

\paragraph{Adversarial regime} A direct application of Theorem~1 of \cite{zimmert2021tsallis}, upper-bounds the pseudo-regret for each player $i \in \{-1,1\}$ as
\begin{equation*}
  \max_{k \in [K]} \E\big[R_{i,T}(k)\big]  \leq 4 \sqrt{KT} + 1 \,.
\end{equation*}
Combining the about bounds with the reduction from MAB to DB  (Theorem~\ref{thm:reduction}) yields the adversarial pseudo-regret upper-bound
\[
  \E\big[ R_T(k)\big] = \frac{1}{2}\E\big[ R_{-1,T}(k) + R_{+1,T}(k)\big] \leq 4\sqrt{KT} + 1.
\]

\paragraph{Adversarial regime with a self-bounding constraint} Our self-bounding constraint is slightly different from that of  \cite{zimmert2021tsallis}, since it involves both players simultaneously. This is necessary so that our gap vector $\Delta$ can recover the standard suboptimality gaps used in stochastic dueling bandits.
Thus, we cannot directly combine their result with our black-box reduction  in this case. However, the proof largely follows their analysis, except that the upper-bounds on the regret of both players must be combined in the middle of their analysis, just before they apply their self-bounding constraint assumption. Thus, we give here only the modification to the proof of their Theorem~1. 

Following their proof of Thm. 1 until their pseudo-regret bound at the top of p. 23, we get for each player $i \in \cI:= \{-1,+1\}$:
\begin{align*}
  \E & \big[R_{i,T}(k)]\\
   &  \leq \sum_{k \neq k^*} \left(\sum_{t=1}^{T} \frac{\sqrt{\E[p_{i,t}(k)]}}{\sqrt{t}}
   + \sum_{t=T_0+1}^T \frac{\E[p_{i,t}(k)]}{4\sqrt{t}}\right) + M  \,,
\end{align*}
where $M:=  \sqrt{T_0} + \frac{3}{4}\sqrt{K} + 15 + 14 K \log(T)$ and $T_0 := \lceil \Delta_{\min}^{-2}/4\rceil$. Together with Theorem~\ref{thm:reduction} and taking the $\max$ over $k$, $\overline{R}_T$ is thus upper-bounded by
\begin{equation*}
  \frac{1}{2} \sum_{i\in \cI} \sum_{k \neq k^*} \left(\sum_{t=1}^{T} \frac{\sqrt{\E[p_{i,t}(k)]}}{\sqrt{t}}
   + \sum_{t=T_0+1}^T \frac{\E[p_{i,t}(k)]}{4\sqrt{t}}\right) + M   \,.
\end{equation*}

Now, applying the self-bounding property~\eqref{eq:self_bounded} we get for any $\lambda \in [0,1]$
\begin{align*}
  \overline{R}_T  \leq \overline{R}_T + \lambda \bigg( \overline{R}_T - \frac{1}{2} \E\bigg[ \sum_{t=1}^{T}  \sum_{k \neq k^*} \big(p_{+1,t}(k) + p_{-1,t}(k)\big) \Delta_k \bigg] + C\bigg)
\end{align*}
Thus, combined with the previous bound using $1+\lambda \leq 2$
\begin{align*}
  &\overline{R}_T 
     \leq \frac{1}{2} \sum_{i\in \cI}\sum_{k \neq k^*} \Bigg(\sum_{t=1}^{T} \bigg( \frac{2\sqrt{\E[p_{i,t}(k)]}}{\sqrt{t}} - \lambda \Delta_k \E[p_{i,t}]\bigg) \\
    & \hspace*{3cm}  + \sum_{t=T_0+1}^T \frac{\E[p_{i,t}(k)]}{2\sqrt{t}}  \Bigg) + 2M  + \lambda C \,. \\
    & \leq \sum_{k\neq k^*}  \Bigg(\sum_{t=1}^{T_0} \max_{z \geq 0} \Big\{\frac{2\sqrt{z}}{\sqrt{t}} - \lambda \Delta_k z\Big\}  \\
    & \hspace*{.5cm} + \sum_{t=T_0+1}^T \max_{z\geq 0} \Big\{\frac{2\sqrt{z}+\frac{1}{2} z}{\sqrt{t}} -\lambda \Delta_i z\Big\}   \Bigg) + 2M + \lambda C
\end{align*}
Now, we are back with the same upper-bound \cite{zimmert2021tsallis} have in the middle of their page 23. Following their analysis by solving the optimization problems, summing over $t$, and optimizing $\lambda$ concludes.
\end{proof}


%% file: arxiv-bdb-22.bbl
\begin{thebibliography}{68}
\providecommand{\natexlab}[1]{#1}
\providecommand{\url}[1]{\texttt{#1}}
\expandafter\ifx\csname urlstyle\endcsname\relax
  \providecommand{\doi}[1]{doi: #1}\else
  \providecommand{\doi}{doi: \begingroup \urlstyle{rm}\Url}\fi

\bibitem[Agarwal et~al.(2021)Agarwal, Agarwal, and
  Patil]{agarwal2021stochastic}
Arpit Agarwal, Shivani Agarwal, and Prathamesh Patil.
\newblock Stochastic dueling bandits with adversarial corruption.
\newblock In \emph{Algorithmic Learning Theory}, pages 217--248. PMLR, 2021.

\bibitem[Ailon et~al.(2014)Ailon, Karnin, and Joachims]{ailon2014reducing}
Nir Ailon, Zohar Karnin, and Thorsten Joachims.
\newblock Reducing dueling bandits to cardinal bandits.
\newblock In \emph{International Conference on Machine Learning}, pages
  856--864. PMLR, 2014.

\bibitem[Alon et~al.(2015)Alon, Cesa-Bianchi, Dekel, and Koren]{Alon+15}
Noga Alon, Nicolo Cesa-Bianchi, Ofer Dekel, and Tomer Koren.
\newblock Online learning with feedback graphs: Beyond bandits.
\newblock In \emph{JMLR Workshop and Conference Proceedings}, volume~40.
  Microtome Publishing, 2015.

\bibitem[Alon et~al.(2017)Alon, Cesa-Bianchi, Gentile, Mannor, Mansour, and
  Shamir]{Alon+17}
Noga Alon, Nicolo Cesa-Bianchi, Claudio Gentile, Shie Mannor, Yishay Mansour,
  and Ohad Shamir.
\newblock Nonstochastic multi-armed bandits with graph-structured feedback.
\newblock \emph{SIAM Journal on Computing}, 46\penalty0 (6):\penalty0
  1785--1826, 2017.

\bibitem[Auer and Chiang(2016)]{auer2016algorithm}
Peter Auer and Chao-Kai Chiang.
\newblock An algorithm with nearly optimal pseudo-regret for both stochastic
  and adversarial bandits.
\newblock In \emph{Conference on Learning Theory}, pages 116--120. PMLR, 2016.

\bibitem[Auer et~al.(2002)Auer, Cesa-Bianchi, and Fischer]{Auer+02}
Peter Auer, Nicolo Cesa-Bianchi, and Paul Fischer.
\newblock Finite-time analysis of the multiarmed bandit problem.
\newblock \emph{Machine learning}, 47\penalty0 (2-3):\penalty0 235--256, 2002.

\bibitem[Auer et~al.(2009)Auer, Jaksch, and Ortner]{ucrl}
Peter Auer, Thomas Jaksch, and Ronald Ortner.
\newblock Near-optimal regret bounds for reinforcement learning.
\newblock In \emph{Advances in neural information processing systems}, pages
  89--96, 2009.

\bibitem[Bai and Jin(2020)]{bai20}
Yu~Bai and Chi Jin.
\newblock Provable self-play algorithms for competitive reinforcement learning.
\newblock In \emph{International Conference on Machine Learning}, pages
  551--560. PMLR, 2020.

\bibitem[Bai et~al.(2020)Bai, Jin, and Yu]{bai+20}
Yu~Bai, Chi Jin, and Tiancheng Yu.
\newblock Near-optimal reinforcement learning with self-play.
\newblock In \emph{Advances in Neural Information Processing Systems}, 2020.

\bibitem[Bengs et~al.(2021)Bengs, Busa-Fekete, El~Mesaoudi-Paul, and
  H{\"u}llermeier]{bengs2021preference}
Viktor Bengs, R{\'o}bert Busa-Fekete, Adil El~Mesaoudi-Paul, and Eyke
  H{\"u}llermeier.
\newblock Preference-based online learning with dueling bandits: A survey.
\newblock \emph{J. Mach. Learn. Res.}, 22:\penalty0 7--1, 2021.

\bibitem[Besbes et~al.(2015)Besbes, Gur, and Zeevi]{besbes+15}
Omar Besbes, Yonatan Gur, and Assaf Zeevi.
\newblock Non-stationary stochastic optimization.
\newblock \emph{Operations research}, 63\penalty0 (5):\penalty0 1227--1244,
  2015.

\bibitem[Brost et~al.(2016)Brost, Seldin, Cox, and Lioma]{Brost+16}
Brian Brost, Yevgeny Seldin, Ingemar~J. Cox, and Christina Lioma.
\newblock Multi-dueling bandits and their application to online ranker
  evaluation.
\newblock \emph{CoRR}, abs/1608.06253, 2016.

\bibitem[Bubeck and Slivkins(2012)]{bubeck2012best}
S{\'e}bastien Bubeck and Aleksandrs Slivkins.
\newblock The best of both worlds: Stochastic and adversarial bandits.
\newblock In \emph{Conference on Learning Theory}, pages 42--1. JMLR Workshop
  and Conference Proceedings, 2012.

\bibitem[Busa-Fekete et~al.(2013)Busa-Fekete, Szorenyi, Cheng, Weng, and
  H{\"u}llermeier]{Busa_top}
R{\'o}bert Busa-Fekete, Balazs Szorenyi, Weiwei Cheng, Paul Weng, and Eyke
  H{\"u}llermeier.
\newblock Top-k selection based on adaptive sampling of noisy preferences.
\newblock In \emph{International Conference on Machine Learning}, pages
  1094--1102, 2013.

\bibitem[Busa-Fekete et~al.(2014)Busa-Fekete, Sz{\"o}r{\'e}nyi, and
  H{\"u}llermeier]{Busa_aaai}
R{\'o}bert Busa-Fekete, Bal{\'a}zs Sz{\"o}r{\'e}nyi, and Eyke H{\"u}llermeier.
\newblock P{AC} rank elicitation through adaptive sampling of stochastic
  pairwise preferences.
\newblock In \emph{AAAI}, pages 1701--1707, 2014.

\bibitem[Chen and Frazier(2017)]{WeakDB}
Bangrui Chen and Peter~I Frazier.
\newblock Dueling bandits with weak regret.
\newblock \emph{arXiv preprint arXiv:1706.04304}, 2017.

\bibitem[Chen et~al.(2018)Chen, Li, and Mao]{ChenSoda+18}
Xi~Chen, Yuanzhi Li, and Jieming Mao.
\newblock A nearly instance optimal algorithm for top-k ranking under the
  multinomial logit model.
\newblock In \emph{Proceedings of the Twenty-Ninth Annual ACM-SIAM Symposium on
  Discrete Algorithms}, pages 2504--2522. SIAM, 2018.

\bibitem[Chen et~al.(2019)Chen, Lee, Luo, and Wei]{luo+19}
Yifang Chen, Chung-Wei Lee, Haipeng Luo, and Chen-Yu Wei.
\newblock A new algorithm for non-stationary contextual bandits: Efficient,
  optimal, and parameter-free.
\newblock \emph{In Proceedings of the 32nd Conference on Learning Theory},
  99:\penalty0 1--30, 2019.

\bibitem[Ding et~al.(2013)Ding, Qin, Zhang, and Liu]{budget3}
Wenkui Ding, Tao Qin, Xu-Dong Zhang, and Tie-Yan Liu.
\newblock Multi-armed bandit with budget constraint and variable costs.
\newblock In \emph{Twenty-Seventh AAAI Conference on Artificial Intelligence},
  2013.

\bibitem[Dud{\'\i}k et~al.(2015)Dud{\'\i}k, Hofmann, Schapire, Slivkins, and
  Zoghi]{CDB}
Miroslav Dud{\'\i}k, Katja Hofmann, Robert~E Schapire, Aleksandrs Slivkins, and
  Masrour Zoghi.
\newblock Contextual dueling bandits.
\newblock In \emph{Conference on Learning Theory}, pages 563--587, 2015.

\bibitem[Falahatgar et~al.(2017)Falahatgar, Hao, Orlitsky, Pichapati, and
  Ravindrakumar]{falahatgar_nips}
Moein Falahatgar, Yi~Hao, Alon Orlitsky, Venkatadheeraj Pichapati, and Vaishakh
  Ravindrakumar.
\newblock Maxing and ranking with few assumptions.
\newblock In \emph{Advances in Neural Information Processing Systems}, pages
  7063--7073, 2017.

\bibitem[Falahatgar et~al.(2018)Falahatgar, Jain, Orlitsky, Pichapati, and
  Ravindrakumar]{falahatgar2}
Moein Falahatgar, Ayush Jain, Alon Orlitsky, Venkatadheeraj Pichapati, and
  Vaishakh Ravindrakumar.
\newblock The limits of maxing, ranking, and preference learning.
\newblock In \emph{International Conference on Machine Learning}, pages
  1427--1436. PMLR, 2018.

\bibitem[Gajane et~al.(2015)Gajane, Urvoy, and Cl{\'e}rot]{Adv_DB}
Pratik Gajane, Tanguy Urvoy, and Fabrice Cl{\'e}rot.
\newblock A relative exponential weighing algorithm for adversarial
  utility-based dueling bandits.
\newblock In \emph{Proceedings of the 32nd International Conference on Machine
  Learning}, pages 218--227, 2015.

\bibitem[Gupta et~al.(2019)Gupta, Koren, and Talwar]{gupta2019better}
Anupam Gupta, Tomer Koren, and Kunal Talwar.
\newblock Better algorithms for stochastic bandits with adversarial
  corruptions.
\newblock In \emph{Conference on Learning Theory}, pages 1562--1578. PMLR,
  2019.

\bibitem[Gupta and Saha(2021)]{gupta2021optimal}
Shubham Gupta and Aadirupa Saha.
\newblock Optimal and efficient dynamic regret algorithms for non-stationary
  dueling bandits.
\newblock \emph{arXiv preprint arXiv:2111.03917}, 2021.

\bibitem[Immorlica et~al.(2019)Immorlica, Sankararaman, Schapire, and
  Slivkins]{budget1}
Nicole Immorlica, Karthik~Abinav Sankararaman, Robert Schapire, and Aleksandrs
  Slivkins.
\newblock Adversarial bandits with knapsacks.
\newblock In \emph{2019 IEEE 60th Annual Symposium on Foundations of Computer
  Science (FOCS)}, pages 202--219. IEEE, 2019.

\bibitem[Jamieson et~al.(2015)Jamieson, Katariya, Deshpande, and
  Nowak]{SparseDB}
Kevin Jamieson, Sumeet Katariya, Atul Deshpande, and Robert Nowak.
\newblock Sparse dueling bandits.
\newblock In \emph{Artificial Intelligence and Statistics}, pages 416--424.
  PMLR, 2015.

\bibitem[Kanade et~al.(2009)Kanade, McMahan, and Bryan]{kanade09}
Varun Kanade, H~Brendan McMahan, and Brent Bryan.
\newblock Sleeping experts and bandits with stochastic action availability and
  adversarial rewards.
\newblock 2009.

\bibitem[Komiyama et~al.(2015)Komiyama, Honda, Kashima, and
  Nakagawa]{Komiyama+15}
Junpei Komiyama, Junya Honda, Hisashi Kashima, and Hiroshi Nakagawa.
\newblock Regret lower bound and optimal algorithm in dueling bandit problem.
\newblock In \emph{COLT}, pages 1141--1154, 2015.

\bibitem[Komiyama et~al.(2016)Komiyama, Honda, and Nakagawa]{Komiyama+16}
Junpei Komiyama, Junya Honda, and Hiroshi Nakagawa.
\newblock Copeland dueling bandit problem: Regret lower bound, optimal
  algorithm, and computationally efficient algorithm.
\newblock \emph{arXiv preprint arXiv:1605.01677}, 2016.

\bibitem[Kumagai(2017)]{ContDB}
Wataru Kumagai.
\newblock Regret analysis for continuous dueling bandit.
\newblock In \emph{Advances in Neural Information Processing Systems}, 2017.

\bibitem[Lattimore and Szepesv{\'a}ri(2019)]{pm1}
Tor Lattimore and Csaba Szepesv{\'a}ri.
\newblock An information-theoretic approach to minimax regret in partial
  monitoring.
\newblock In \emph{Conference on Learning Theory}, pages 2111--2139. PMLR,
  2019.

\bibitem[Lin et~al.(2014)Lin, Abrahao, Kleinberg, Lui, and Chen]{pm3}
Tian Lin, Bruno Abrahao, Robert Kleinberg, John Lui, and Wei Chen.
\newblock Combinatorial partial monitoring game with linear feedback and its
  applications.
\newblock In \emph{International Conference on Machine Learning}, pages
  901--909. PMLR, 2014.

\bibitem[Lykouris et~al.(2018)Lykouris, Mirrokni, and
  Paes~Leme]{lykouris2018stochastic}
Thodoris Lykouris, Vahab Mirrokni, and Renato Paes~Leme.
\newblock Stochastic bandits robust to adversarial corruptions.
\newblock In \emph{Proceedings of the 50th Annual ACM SIGACT Symposium on
  Theory of Computing}, pages 114--122, 2018.

\bibitem[Mannor et~al.(2014)Mannor, Perchet, and Stoltz]{pm2}
Shie Mannor, Vianney Perchet, and Gilles Stoltz.
\newblock Set-valued approachability and online learning with partial
  monitoring.
\newblock \emph{The Journal of Machine Learning Research}, 15\penalty0
  (1):\penalty0 3247--3295, 2014.

\bibitem[Negahban et~al.(2017)Negahban, Oh, and Shah]{shah17}
Sahand Negahban, Sewoong Oh, and Devavrat Shah.
\newblock Rank centrality: Ranking from pairwise comparisons.
\newblock \emph{Operations Research}, 65\penalty0 (1):\penalty0 266--287, 2017.

\bibitem[Neu and Valko(2014)]{neu14}
Gergely Neu and Michal Valko.
\newblock Online combinatorial optimization with stochastic decision sets and
  adversarial losses.
\newblock In \emph{Advances in Neural Information Processing Systems}, pages
  2780--2788, 2014.

\bibitem[Ng et~al.(2006)Ng, Coates, Diel, Ganapathi, Schulte, Tse, Berger, and
  Liang]{ng06}
Andrew~Y Ng, Adam Coates, Mark Diel, Varun Ganapathi, Jamie Schulte, Ben Tse,
  Eric Berger, and Eric Liang.
\newblock Autonomous inverted helicopter flight via reinforcement learning.
\newblock In \emph{Experimental robotics IX}, pages 363--372. Springer, 2006.

\bibitem[Niranjan and Rajkumar(2017)]{niranjan2017}
UN~Niranjan and Arun Rajkumar.
\newblock Inductive pairwise ranking: going beyond the n log (n) barrier.
\newblock In \emph{Proceedings of the AAAI Conference on Artificial
  Intelligence}, volume~31, 2017.

\bibitem[Pike-Burke et~al.(2018)Pike-Burke, Agrawal, Szepesvari, and
  Grunewalder]{delay2}
Ciara Pike-Burke, Shipra Agrawal, Csaba Szepesvari, and Steffen Grunewalder.
\newblock Bandits with delayed, aggregated anonymous feedback.
\newblock In \emph{International Conference on Machine Learning}, pages
  4105--4113, 2018.

\bibitem[Ren et~al.(2018)Ren, Liu, and Shroff]{Ren+18}
Wenbo Ren, Jia Liu, and Ness~B Shroff.
\newblock P{AC} ranking from pairwise and listwise queries: Lower bounds and
  upper bounds.
\newblock \emph{arXiv preprint arXiv:1806.02970}, 2018.

\bibitem[Saha(2021)]{S21}
Aadirupa Saha.
\newblock Optimal algorithms for stochastic contextual dueling bandits.
\newblock In \emph{Advances in Neural Information Processing Systems}, 2021.

\bibitem[Saha and Gaillard(2021)]{SG21dbaa}
Aadirupa Saha and Pierre Gaillard.
\newblock Dueling bandits with adversarial sleeping.
\newblock \emph{Advances in Neural Information Processing Systems}, 34, 2021.

\bibitem[Saha and Gopalan(2018{\natexlab{a}})]{SG18}
Aadirupa Saha and Aditya Gopalan.
\newblock Battle of bandits.
\newblock In \emph{Uncertainty in Artificial Intelligence}, 2018{\natexlab{a}}.

\bibitem[Saha and Gopalan(2018{\natexlab{b}})]{SGrank18}
Aadirupa Saha and Aditya Gopalan.
\newblock Active ranking with subset-wise preferences.
\newblock \emph{International Conference on Artificial Intelligence and
  Statistics (AISTATS)}, 2018{\natexlab{b}}.

\bibitem[Saha and Gopalan(2019{\natexlab{a}})]{SG19}
Aadirupa Saha and Aditya Gopalan.
\newblock Combinatorial bandits with relative feedback.
\newblock In \emph{Advances in Neural Information Processing Systems},
  2019{\natexlab{a}}.

\bibitem[Saha and Gopalan(2019{\natexlab{b}})]{SGwin18}
Aadirupa Saha and Aditya Gopalan.
\newblock {PAC Battling Bandits in the Plackett-Luce Model}.
\newblock In \emph{Algorithmic Learning Theory}, pages 700--737,
  2019{\natexlab{b}}.

\bibitem[Saha and Gopalan(2020)]{SGinst20}
Aadirupa Saha and Aditya Gopalan.
\newblock From pac to instance-optimal sample complexity in the plackett-luce
  model.
\newblock In \emph{International Conference on Machine Learning}, pages
  8367--8376. PMLR, 2020.

\bibitem[Saha and Krishnamurthy(2021)]{SK21}
Aadirupa Saha and Akshay Krishnamurthy.
\newblock Efficient and optimal algorithms for contextual dueling bandits under
  realizability.
\newblock \emph{arXiv preprint arXiv:2111.12306}, 2021.

\bibitem[Saha et~al.(2021)Saha, Koren, and Mansour]{ADB}
Aadirupa Saha, Tomer Koren, and Yishay Mansour.
\newblock Adversarial dueling bandits.
\newblock In \emph{International Conference on Machine Learning}, pages
  9235--9244. PMLR, 2021.

\bibitem[Sui et~al.(2018)Sui, Zoghi, Hofmann, and Yue]{sui2018advancements}
Yanan Sui, Masrour Zoghi, Katja Hofmann, and Yisong Yue.
\newblock Advancements in dueling bandits.
\newblock In \emph{IJCAI}, pages 5502--5510, 2018.

\bibitem[Sz{o}r{e}nyi et~al.(2015)Sz{o}r{e}nyi, Busa-Fekete, Paul, and
  H{u}llermeier]{Busa_pl}
Bal{a}zs Sz{o}r{e}nyi, R{o}bert Busa-Fekete, Adil Paul, and Eyke H{u}llermeier.
\newblock Online rank elicitation for plackett-luce: A dueling bandits
  approach.
\newblock In \emph{Advances in Neural Information Processing Systems}, pages
  604--612, 2015.

\bibitem[Talebi and Maillard(2018)]{talebi18}
Mohammad~Sadegh Talebi and Odalric-Ambrym Maillard.
\newblock Variance-aware regret bounds for undiscounted reinforcement learning
  in mdps.
\newblock \emph{arXiv preprint arXiv:1803.01626}, 2018.

\bibitem[Thune et~al.(2019)Thune, Cesa-Bianchi, and Seldin]{delay3}
Tobias~Sommer Thune, Nicolo Cesa-Bianchi, and Yevgeny Seldin.
\newblock Nonstochastic multiarmed bandits with unrestricted delays.
\newblock In \emph{Advances in Neural Information Processing Systems}, pages
  6541--6550, 2019.

\bibitem[Vernade et~al.(2018)Vernade, Carpentier, Lattimore, Zappella, Ermis,
  and Brueckner]{delay1}
Claire Vernade, Alexandra Carpentier, Tor Lattimore, Giovanni Zappella, Beyza
  Ermis, and Michael Brueckner.
\newblock Linear bandits with stochastic delayed feedback.
\newblock \emph{arXiv preprint arXiv:1807.02089}, 2018.

\bibitem[Wei and Luo(2021)]{wei21}
Chen-Yu Wei and Haipeng Luo.
\newblock Non-stationary reinforcement learning without prior knowledge: An
  optimal black-box approach.
\newblock \emph{In Proceedings of the 32nd International Conference on Learning
  Theory}, 2021.

\bibitem[Wu and Liu(2016)]{DTS}
Huasen Wu and Xin Liu.
\newblock Double {T}hompson sampling for dueling bandits.
\newblock In \emph{Advances in Neural Information Processing Systems}, pages
  649--657, 2016.

\bibitem[Xie et~al.(2020)Xie, Chen, Wang, and Yang]{xie+20}
Qiaomin Xie, Yudong Chen, Zhaoran Wang, and Zhuoran Yang.
\newblock Learning zero-sum simultaneous-move markov games using function
  approximation and correlated equilibrium.
\newblock In \emph{Conference on Learning Theory}, 2020.

\bibitem[Yue and Joachims(2009)]{Yue+09}
Yisong Yue and Thorsten Joachims.
\newblock Interactively optimizing information retrieval systems as a dueling
  bandits problem.
\newblock In \emph{Proceedings of the 26th Annual International Conference on
  Machine Learning}, pages 1201--1208. ACM, 2009.

\bibitem[Yue and Joachims(2011)]{BTM}
Yisong Yue and Thorsten Joachims.
\newblock Beat the mean bandit.
\newblock In \emph{Proceedings of the 28th International Conference on Machine
  Learning (ICML-11)}, pages 241--248, 2011.

\bibitem[Yue et~al.(2012)Yue, Broder, Kleinberg, and Joachims]{Yue+12}
Yisong Yue, Josef Broder, Robert Kleinberg, and Thorsten Joachims.
\newblock The $k$-armed dueling bandits problem.
\newblock \emph{Journal of Computer and System Sciences}, 78\penalty0
  (5):\penalty0 1538--1556, 2012.

\bibitem[Zhou and Tomlin(2018)]{budget2}
Datong~P Zhou and Claire~J Tomlin.
\newblock Budget-constrained multi-armed bandits with multiple plays.
\newblock In \emph{Thirty-Second AAAI Conference on Artificial Intelligence},
  2018.

\bibitem[Zimmert and Seldin(2020)]{zimmert2020optimal}
Julian Zimmert and Yevgeny Seldin.
\newblock An optimal algorithm for adversarial bandits with arbitrary delays.
\newblock In \emph{International Conference on Artificial Intelligence and
  Statistics}, pages 3285--3294. PMLR, 2020.

\bibitem[Zimmert and Seldin(2021)]{zimmert2021tsallis}
Julian Zimmert and Yevgeny Seldin.
\newblock Tsallis-inf: An optimal algorithm for stochastic and adversarial
  bandits.
\newblock \emph{J. Mach. Learn. Res.}, 22:\penalty0 28--1, 2021.

\bibitem[Zoghi et~al.(2014{\natexlab{a}})Zoghi, Whiteson, Munos, Rijke,
  et~al.]{Zoghi+14RUCB}
Masrour Zoghi, Shimon Whiteson, Remi Munos, Maarten~de Rijke, et~al.
\newblock Relative upper confidence bound for the $k$-armed dueling bandit
  problem.
\newblock In \emph{JMLR Workshop and Conference Proceedings}, number~32, pages
  10--18. JMLR, 2014{\natexlab{a}}.

\bibitem[Zoghi et~al.(2014{\natexlab{b}})Zoghi, Whiteson, De~Rijke, and
  Munos]{Zoghi+14RCS}
Masrour Zoghi, Shimon~A Whiteson, Maarten De~Rijke, and Remi Munos.
\newblock Relative confidence sampling for efficient on-line ranker evaluation.
\newblock In \emph{Proceedings of the 7th ACM international conference on Web
  search and data mining}, pages 73--82. ACM, 2014{\natexlab{b}}.

\bibitem[Zoghi et~al.(2015{\natexlab{a}})Zoghi, Karnin, Whiteson, and
  De~Rijke]{Zoghi+15}
Masrour Zoghi, Zohar~S Karnin, Shimon Whiteson, and Maarten De~Rijke.
\newblock Copeland dueling bandits.
\newblock In \emph{Advances in Neural Information Processing Systems}, pages
  307--315, 2015{\natexlab{a}}.

\bibitem[Zoghi et~al.(2015{\natexlab{b}})Zoghi, Whiteson, and
  de~Rijke]{Zoghi+15MRUCB}
Masrour Zoghi, Shimon Whiteson, and Maarten de~Rijke.
\newblock Mergerucb: A method for large-scale online ranker evaluation.
\newblock In \emph{Proceedings of the Eighth ACM International Conference on
  Web Search and Data Mining}, pages 17--26. ACM, 2015{\natexlab{b}}.

\end{thebibliography}
